\documentclass[sigconf]{aamas} 

\usepackage{balance} 
\usepackage[ruled, linesnumbered]{algorithm2e}
\usepackage{subcaption}
\usepackage{soul}
\newtheorem{thm}{Theorem}[section]
\newtheorem{cor}[thm]{Corollary}
\newtheorem{lem}[thm]{Lemma}
\newtheorem{cond}[thm]{Condition}
\newtheorem{assume}[thm]{Assumption}
\newtheorem{prop}[thm]{Proposition}
\newtheorem{defn}[thm]{Definition}
\newcommand{\ouralgo}{\textsc{MF-RMAB}}
\newcommand{\ourfair}{\textsc{Merit Fair}}

\setcopyright{ifaamas}
\acmConference[AAMAS '24]{Proc.\@ of the 23rd International Conference
on Autonomous Agents and Multiagent Systems (AAMAS 2024)}{May 6 -- 10, 2024}
{Auckland, New Zealand}{N.~Alechina, V.~Dignum, M.~Dastani, J.S.~Sichman (eds.)}
\copyrightyear{2024}
\acmYear{2024}
\acmDOI{}
\acmPrice{}
\acmISBN{}

\title{Fairness of Exposure in Online Restless Multi-armed Bandits}

\author{Archit Sood}
\affiliation{
  \institution{Indian Institute of Technology Ropar}
  \city{Rupnagar}
  \country{India}}
\email{2020mcb1230@iitrpr.ac.in}

\author{Shweta Jain}
\affiliation{
  \institution{Indian Institute of Technology Ropar}
  \city{Rupnagar}
  \country{India}}
\email{shwetajain@iitrpr.ac.in}

\author{Sujit Gujar}
\affiliation{
  \institution{International Institute of Information Technology Hyderabad}
  \city{Hyderabad}
  \country{India}}
\email{sujit.gujar@iiit.ac.in}

\begin{abstract}
    Restless multi-armed bandits (RMABs) generalize the multi-armed bandits where each arm exhibits Markovian behavior and transitions according to their transition dynamics. Solutions to RMAB exist for both offline and online cases. However, they do not consider the distribution of pulls among the arms. Studies have shown that optimal policies lead to  unfairness, where some arms are not exposed enough. Existing works in fairness in RMABs focus heavily on the offline case, which diminishes their application in real-world scenarios where the environment is largely unknown. In the online scenario, we propose the first fair RMAB framework, where each arm receives pulls in proportion to its merit. We define the merit of an arm as a function of its stationary reward distribution. We prove that our algorithm achieves sublinear fairness regret in the single pull case  $O(\sqrt{T\ln T})$,  with $T$ being the total number of episodes. Empirically, we show that our algorithm performs well in the multi-pull scenario as well.  
\end{abstract}


\newcommand{\BibTeX}{\rm B\kern-.05em{\sc i\kern-.025em b}\kern-.08em\TeX}
\theoremstyle{acmplain}

\begin{document}


\pagestyle{fancy}
\fancyhead{}


\maketitle 


\section{Introduction}
\emph{Bandit algorithms} have gained increasing popularity in recent years in modelling decision-making problems where the decision-maker interacts with an unknown environment. At each round, the decision maker has to choose a subset of choices or arms, resulting in certain rewards for the decision maker. The goal, therefore, is to focus both on learning the environment while also choosing arms in such a way as to maximize the total reward. Bandit algorithms are useful in various applications ranging from marketing~\cite{rothschild1974two}, crowdsourcing ~\cite{abraham2013adaptive}, resource allocations~\cite{berry1985bandit}, healthcare~\cite{biswas2021learn, killian2023equitable}, etc. 

\emph{Restless Multi-Armed Bandits} (RMABs) are a class of MABs where each arm has a Markov Decision Process (MDP) associated with it. Each arm has its own states, actions, transition dynamics, and reward functions. In RMABs, the number of actions associated with each arm could be two or more two~\cite{hodge2015asymptotic, killian2021beyond,killian2021q}. In line with most of the literature, we consider two actions per arm -- to pull or not pull the arm ~\cite{whittle1988restless,jung2019regret,biswas2021learn,wang2023optimistic}. A slightly different line of work is around rested bandits, where at each round, only the arms that receive a pull transition from one state to another. However in restless bandits, all the arms transition from one state to the next state, irrespective of whether they are pulled or not. In RMABs, since arms transition irrespective of whether pulled or not, they exhibit different transition dynamics for the action taken (pull/not pull). It is this \emph{restless} nature of the arms that makes RMABs applicable to many domains such as network scheduling ~\cite{modi2019transfer}, anti-poaching ~\cite{qian2016restless}, healthcare ~\cite{mate2022field}, etc.

Recently, a lot of works have been using restless bandits to model preventive interventions in public healthcare scenarios ~\cite{biswas2021learn, herlihy2023planning, killian2023equitable, mate2020collapsing, killian2023adherence}. Here, each arm corresponds to a patient or a beneficiary, whereas the decision maker may be a public hospital or agency that provides intervention to these patients. Providing an intervention to each patient is considered equivalent to pulling an arm.
Because of the limited resources, such as the limited number of workers working at the hospital or agency, intervention can be provided to only a limited number of patients at a time. When intervened, each patient's state (good/bad) changes depending on whether they received the intervention or not. It is essential to look for policies that provide the maximum benefit to the public, i.e., policies that ensure most patients remain in a good state for longer. 

Some of the preliminary literature on restless bandits assumes that the transition dynamics for each arm's MDP are known i.e. \emph{offline setting}. In general, it has been proven that finding an optimal policy for RMABs even when the transition dynamics are known is PSPACE-hard ~\cite{papadimitriou1994complexity}. ~\citet{whittle1988restless} provided an indexing policy that is asymptotically optimal given that the arms are indexable ~\cite{weber1990index,akbarzadeh2019restless}. Recently, there have been some advancements in \emph{online setting}, where these transition dynamics are not known and are learned over a period of time. To this, researchers have explored Thompson Sampling~\cite{jung2019regret,jung2019thompson}, Upper Confidence Bound (UCB)~\cite{ortner2012regret,wang2023optimistic} and Model-free Q-learning approaches ~\cite{fu2019towards,avrachenkov2022whittle,biswas2021learn}, among many others. However, all these approaches focus only on finding the optimal policy -- leading to some arms being completely ignored~\cite{prins2020incorporating}. As in our running example where arms model patients, this represents a major problem: the optimal policies would focus only on patients who require the most interventions and ignore the patients who rarely need interventions. However, in public healthcare, it becomes important to focus on all kinds of patients and not just a few, so as to provide unbiased healthcare to society.

Fairness in RMABs is an upcoming research direction that seeks to address this problem. Current work on fairness in RMABs primarily focuses on the offline setting~\cite{herlihy2023planning,li2023avoiding,prins2020incorporating,mate2021risk}. Only a few fairness notions are prevalent in the literature on the restless bandits setting. The first notion ensures that each arm should be pulled at least once every fixed time period~\cite{li2022efficient}, and the second notion provides a lower bound on each arm's probability of receiving a pull by user-defined constants~\cite{herlihy2023planning}. These two fairness notions do not consider the heterogeneity of arms and introduce universal fairness constraints. ~\citet{li2023avoiding} talk about allocation vectors and that a more rewarding allocation vector should be chosen with a higher probability than a worse allocation vector. It still does not guarantee a minimum exposure guarantee to each arm and in worst case scenario can lead to some arms barely receiving pulls. Thus, it is a much weaker fairness notion. ~\citet{killian2023equitable} considers equitable group fairness where each arm belongs to a particular group, and the fairness notion provides enough exposure to each group so as to ensure equal outcomes. In this approach as well, the pulls may be still distributed unfairly within a group. Moreover, it requires apriori knowledge of which group an arm belongs to. Further, all the above works assume that the transition dynamics are known beforehand and construct their policies based on this assumption. To our knowledge, only~\citet{li2022efficient} explore fairness in online RMABs, where they consider equality focused fairness with universal fairness constraints. Therefore, there is a need to design online RMABs with stronger fairness guarantees.

We propose that fairness should be defined with respect to the steady-state distribution of the MDPs of the arms. The steady-state is important because it talks about the state that the arms would end up in when we run a policy sufficiently long, which would happen in a real-world setting. We define the goodness of an arm by the difference in the steady-state distribution of the arm being in the \emph{good} state when we always pull the arm, compared to when we never pull the arm. This metric is indicative of how much an arm benefits from intervention. If we did not use the difference and instead just chose the steady state probability of being in the good state to be the merit of an arm, then it may so happen that this arm would have ended up in the good state without even needing intervention, and we would have wasted our limited budget on an arm which did not necessarily need it. Taking the difference helps us identify the arms that are impacted the most by our help. Once we have a notion of the merit of an arm, one can define meritocratic fairness similar to that of \cite{wang2021fairness}. We call our notion of fairness as \emph{\ourfair\ } where each arm is pulled with probability proportional to how much benefit we obtain at a steady state by pulling the arm. \ourfair\ notion is better than the existing notions in RMABs because it does not use any universal fairness constant like in \cite{li2022efficient, herlihy2023planning}, but rather provides exposure to each arm based on its merit.

Generally, a multiple pull setting is considered in RMABs which means that at each round, the public hospital makes a call to K patients to be intervened. In this paper, for theoretical analysis, we primarily focus on single pull settings for the following reasons. Meritocratic fairness \cite{wang2021fairness} has been specially designed for pulling a single arm at each round. It is not clear how such merit-based fairness can be extended beyond multiple pulls. We show that extending merit-based fairness to multiple pulls will require some technical assumptions, which may not necessarily hold in practice. Even if such technical requirements are satisfied, with multiple pull RMAB, quantifying fairness is challenging. For example, suppose we have a budget to pull all the arms, then we will not be able to learn transition probabilities of action `not pulling' the arm effectively unless we forgo pulling an arm despite having the available budget. Hence, in this work, we focus on single-pull settings for theoretical aspects, though our algorithm can be extended to multiple pulls. Therefore, we study its efficacy on multiple pull settings on synthetic and real-world datasets. To the best of our knowledge, we are the first one to provide theoretical results in the online fair RMAB problem, and also the first one to extend the Fairness of Exposure~\cite{wang2021fairness} notion to Restless bandits. 
Our core contributions are:
\begin{itemize}
    \item We extend the exposure fairness notion of ~\citet{wang2021fairness} to an online RMAB setting.
    \item We provide a sublinear bound on fairness regret when only one arm is pulled and provide experimental results on both single pull and multiple pull cases.
    \item We validate the efficacy of our proposed algorithm on synthetic and real-world datasets ~\cite{kang2016modelling}.
\end{itemize}

The paper is structured as follows. Section 2 summarizes the required literature background. In Section 3, we define preliminaries and notations used consistently throughout the paper. Section 4 provides details to our approach and Section 5 provides its theoretical guarantees. We validate the efficacy of our method on real and synthetic data in Section 6.

\section{Related Work}
\subsection{Restless Bandits}
Whittle ~\cite{whittle1988restless} provides an asymptotically optimal indexing approach for RMABs when the transition dynamics are known. When the transition dynamics are unknown, previous works include Thompson Sampling based approaches ~\cite{jung2019regret,jung2019thompson} which construct policies based on a prior and then update the prior according to the feedback from the environment; Upper Confidence Bound based approaches ~\cite{ortner2012regret, wang2023optimistic} which maintain a confidence region for the transition dynamics and construct policies based on this region; and model-free approaches ~\cite{fu2019towards,avrachenkov2022whittle,biswas2021learn} which skip learning the environment dynamics altogether and only focus on finding the Whittle indices to determine the optimal action to take in each round. In this paper, we will be using upper confidence bound-based approaches to learn the transition dynamics.
\subsection{Fairness in MAB}
There have been multiple approaches to address fairness in multi-armed bandits. Here, we discuss some fairness notions available in the literature concerning stochastic bandits setting ~\cite{joseph2016fairness, heidari2018preventing, li2019combinatorial,chen2020fair, patil2021achieving,wang2021fairness}. ~\citet{joseph2016fairness} use the idea that a better arm is always selected with greater or equal probability than a worse arm. This fairness notion, while good for the optimal arm, may lead to many arms getting very minimal exposure. Later, some works define fairness in terms of providing minimum/maximum amount of exposure to each of the arms ~\cite{heidari2018preventing,li2019combinatorial,chen2020fair}. On similar lines, ~\citet{patil2021achieving} achieve anytime fairness guarantee, i.e., at each round, every arm will be pulled at least a fixed predetermined fraction of times. This notion requires the user to input a predetermined vector quantifying the exposure required for each arm. Later, ~\citet{wang2021fairness} extended the fairness notion to proportional fairness which requires that each arm be pulled with a probability proportional to its merit. The authors compare their algorithm with an optimal fair benchmark, which is aware of the true merit of each arm and does not need to learn anything. 

\subsection{Fairness in RMAB}
Fairness in Restless Bandits has only recently started to be explored. Offline algorithms, which assume full knowledge of the transition dynamics of each arm, include ~\cite{herlihy2023planning} which ensures that the probability of each arm being pulled is in the interval $[l,u]$ where $l,u$ are user-defined parameters. ~\citet{li2023avoiding} provide the fairness notion that an allocation vector should always be chosen with greater or more probability than a worse allocation vector. ~\citet{prins2020incorporating} consider RMABs where only the states of the arms that are pulled are observable and the rest are not, and their fairness notion is that each arm should be pulled at least once every fixed period of time. ~\citet{mate2021risk} show that their reward function can be shaped so as to incentivize fairness. Finally, ~\citet{killian2023equitable}  consider group fairness and frames the problem so as to find the allocation among the groups that maximizes the total social welfare. In an online setting, ~\citet{li2022efficient} provide a Q-learning Whittle Indexed based algorithm that operates similarly to ~\cite{patil2021achieving, chen2020fair}. Their fairness constraint that each arm receives a pull every fixed time period does not take into account the heterogeneity of the arms and they do not discuss any notions of regret while learning. Our work combines the fairness notion of ~\cite{wang2021fairness}, the Online RMAB approach of ~\cite{wang2023optimistic}, and the steady state reward formulation of ~\cite{herlihy2023planning}, to provide a novel analysis on regret in online Fair RMABs.   

\section{Preliminaries}
An RMAB problem is defined by a set of $N$ independent arms. Each arm $i \in [N]$ is characterized by a Markov Decision Process (MDP) given by $(\mathcal{S},\mathcal{A},\mathcal{R},P_i)$ with $\mathcal{S}, \mathcal{A}$, and $\mathcal{R}: \mathcal{S} \xrightarrow{} \mathbb{R}$ denoting the state space, action space, and reward function respectively. $P_i: \mathcal{S} \times \mathcal{A} \times \mathcal{S} \xrightarrow[]{} [0,1]$ is the transition probablity matrix for arm $i$. In the traditional RMAB setting, $\mathcal{S}, \mathcal{A}, \mathcal{R}$ are taken to be common among all the arms, and each arm differs only by their transition matrix $P_i$. The states are assumed to be fully observable. The action the decision-maker takes is governed by a policy $\pi$. The total number of episodes is $T$, where a policy $\pi^t$ is fixed for $t \leq T$, and is run for $H$ timesteps, where $H$ is the time horizon of an episode. For each timestep $h \in H$ in episode $t$, the decision-maker has to select $K \leq N$ arms according to $\pi^t$, where $K$ is the budget.  

In our setting, we assume $\mathcal{S} \coloneqq \{0,1\}$, where 0 denotes \emph{bad} state and 1 denotes \emph{good} state. There are two possible actions, i.e. $\mathcal{A} \coloneqq \{1,0\}$, indicating whether an arm is pulled or not respectively. We take the reward as $\mathcal{R}(s) = s \ \forall s \in \mathcal{S}$, i.e., we receive a reward of 1 when the arm is in the good state and a reward of 0 when the arm is in the bad state. As the state and action spaces are discrete and finite, we can view the transition matrix $P_i$ as a $|\mathcal{S}| \times |\mathcal{A}| \times |\mathcal{S}|$ dimensional matrix where $P_i(s,a,s') \in [0,1]$ is the probability of transitioning to state $s'$ if action $a$ is taken while in state $s$. 

In an online setting, $P_i$'s are unknown and need to be learned over multiple episodes. 
Let us denote the true transition matrix for an arm $i$ as $P^*_i$. We assume that $P^*_i$'s are \emph{non-degenerate}, i.e., there exists an $\epsilon > 0$ such that $\epsilon \leq P_i^*(s,a,s') \leq 1 - \epsilon \ \ \ \forall i \in [N], a \in \mathcal{A}, s,s \in \mathcal{S}'$.
Intuitively, this means there is no deterministic transition, and there is always a non-zero chance of ending up in any state after action. This is a natural assumption and maps many real-life scenarios such as healthcare applications where there are rarely any deterministic transitions. 

\subsection{\ourfair: Merit-based fairness in RMAB}\label{topic:exposure}

To define meritocratic fairness for the arms, we must define a certain merit of pulling to each arm. As standard in the MAB literature, we refer to it also as a \emph{reward}. 
Let $\mu_i^*$ denote the true reward of arm $i$, which we define formally later. Along a similar line to Wang et al. ~\cite{wang2021fairness}, let us define the Optimal Fair Policy as $Pr^*(K)$, where $Pr^*_i(K)$ is the probability that arm $i$ is among the $K$ chosen out of the $N$ total arms. Observe that $Pr^*_i(N)$ = 1 and that $Pr^*_i(1)$ = $\pi_i^*$, where $\pi^*$ is the probability distribution of being chosen over the arms. Note that $\sum_i \pi_i^* = 1$ irrespective of $K$. Let $g(\cdot)$ be a non-decreasing merit function that maps the reward of the arm to a positive value. Then, for the optimal fair policy, the following equation holds.
\begin{align}\label{eqn:fairness}
    \frac{Pr_i^*(K)}{g(\mu_i^*)} = \frac{Pr_j^*(K)}{g(\mu_j^*)} \ \ \forall i,j \in [N]
\end{align}
This implies that the exposure each arm gets is proportional to its merit. However, observe that if there is an arm $l$ such that $g(\mu_l)>K\times g(\mu_i)\;\forall i\neq l$, the probability of pulling the arm $l$ becomes more than $1$ to ensure desired fairness. Thus, we cannot guarantee such meritocratic fairness for multiple pulls.  Even if we assume $g(\cdot), \ \mu's$ are such that $Pr_i(\cdot)<1\;\forall i$, there are further challenges in learning $P_i$'s for large $K$. Hence, for theoretical analysis, we focus on $K=1$. Note that for $K=1$, Equation (\ref{eqn:fairness}) becomes 
\begin{align}\label{eqn:fairness_K=1}
    \frac{\pi_i^*}{g(\mu_i^*)} = \frac{\pi_j^*}{g(\mu_j^*)} \ \ \forall i,j \in [N]
\end{align}

Let $\pi = \{\pi^t\}_{t=1}^T$ be the policy learnt by our algorithm with $\pi^t$ being the employed policy at episode $t$. We define \emph{Fairness Regret} $FR^T$ as the difference between the optimal fair policy $\pi^*$ and our policy $\pi$ up to episode $T$. Mathematically,
\begin{align}\label{dfn:regret_def}
    FR^T = \sum_{t=1}^T \sum_{i \in [N]} \left | \pi_i^* - \pi_i^t \right |
\end{align}
Here, $\pi_i^t$ denotes the probability of pulling an arm $i$ in episode $t$.
We note that the notion of Reward Regret is also defined in ~\cite{wang2021fairness} which we do not analyze it in this work. This is because \cite{wang2021fairness} discusses the stochastic MAB setting, where the reward of the arm is the actual reward we get after pulling the arm. In our case, we define $\mu_i^*$ in such a way that it only represents the benefit of pulling an arm, and only exists to indicate which arm is better suited for receiving a pull. It does not correlate to the actual reward (given by the reward function $\mathcal{R}$) that we receive after pulling some arm. Therefore, the concept of reward regret does not make much sense in our work. For the sake of the theoretical analysis of the regret, we impose the same conditions on the merit function $g(\cdot)$ as imposed in ~\cite{wang2021fairness}.
\begin{cond} \label{cond:gamma}
    Merit of each arm is positive, i.e., $\min_\mu g(\mu) \geq \gamma$ for some $\gamma$ > 0.
\end{cond}
\begin{cond}\label{cond:L}
    Merit function is L-Lipschitz continuous, i.e., $\forall \mu_1,\mu_2, \\ {|g(\mu_1) - g(\mu_2)|} \leq {L|\mu_1 - \mu_2|}$ for some constant L > 0.
\end{cond}

\section{Methodology}
\subsection{Defining the reward}
We first define a reward that is based on steady state and is indicative of how much intervention an arm requires. Consider the policy discussed by ~\citet{herlihy2023planning} where each arm is pulled with some fixed probability $p_i$. Then this policy can be defined as $\pi_{PF}: \{i \ | \ i\in[N]\} \xrightarrow{}[1-p_i, \ p_i]^N$. Repeated application of this policy will result in a steady state distribution that tells us the probability of an arm being in state 1. Let us denote $f(P_i,p_i)$ to be the steady state probability of arm $i$ being in state 1, when followed a policy $\pi_{PF}$. Once, we acquire these probabilities, the reward of an arm can be naturally defined as: 
\begin{align}\label{defn:reward}    
    \mu_i = f(P_i,1) - f(P_i,0)
\end{align}
This reward signifies the difference at steady state when we always pull arm $i$ as compared to when we never pull that arm. In other words, the reward represents the benefit of pulling an arm in the long run as compared to the loss the algorithm would have incurred if it had not pulled the arm.

We can now calculate the steady state probabilities as follows. 
If the transition matrix for arm $i$ is $P_i$, at steady state, we should have
\begin{align*}
    f(P_i,p_i)[(1-p_i)P_i(1,0,1) + p_i P_i(1,1,1)] \ + \\ (1-f(P_i,p_i))[(1-p_i)P_i(0,0,1) + p_i P_i(0,1,1)] &= f(P_i,p_i)
\end{align*}
After simplifying the above expression, we get,  $ f(P_i,p_i) = $  
\begin{align}\label{fi(pi)}
    \frac{(1-p_i)P_i(0,0,1) + p_iP_i(0,1,1)}{1 - (1-p_i)P_i(1,0,1) - p_iP_i(1,1,1) + (1-p_i)P_i(0,0,1) + p_iP_i(0,1,1)}
\end{align}
We can then simplify reward $\mu_i$ as:
\begin{align*}
    \mu_i = \frac{P_i(0,1,1)}{1-P_i(1,1,1)+P_i(0,1,1)} - \frac{P_i(0,0,1)}{1-P_i(1,0,1)+P_i(0,0,1)}
\end{align*}

After defining the reward of the arms, we can continue towards the fairness notions discussed in Section ~\ref{topic:exposure}. If we were to pull a single arm in every round, then \citet{wang2021fairness} proved that a unique fair policy satisfying Equation (\ref{eqn:fairness}) is given by: $\pi^*_i = \frac{g(\mu_i)}{\sum_{j\in N}g(\mu_j)}$.

We will explain the process of pulling multiple arms later. It is further to be noted that the reward of an arm $i$ only depends on the probability transition matrix $P_i$ which is not known beforehand. The next subsection describes the procedure to learn these probability transition matrices using upper confidence bound techniques.
\subsection{Online RMAB}
As we are in an online setting, we seek to estimate the true transition matrices $P_i^*$'s. We use the Upper Confidence Bound (UCB) approach which maintains an optimistic bound on the true transition matrix corresponding to each state-action-state ~\cite{wang2023optimistic}. The estimated value of the transition probability matrix of an arm for each $(s,a,s')$ can be computed by finding a fraction of the time in which the arm transitioned to state $s'$ from state $s$, when action $a$ was taken. This estimate and the number of times the arm was in state $s$ and action $a$ is considered, and a confidence radius is then created. It could further be proved that with high probability, the true transition matrix would lie within the confidence radius of the estimated transition matrix. We explain the detailed procedure below.

Let $N_i^t(s,a,s')$ be the number of times $(s,a,s')$ transition has been observed for arm $i$ by episode $t$. Further, define $N_i^t(s,a) = \sum_{s'}N_i^t(s,a,s')$ to be the total number of times arm $i$ had been in state $s$ when action $a$ is taken. Then at episode $t$, we estimate the true transition matrix $P_i^*(s,a,s')$ with empirical mean
$$\hat{P}_i^t(s,a,s') \coloneqq \frac{N_i^t(s,a,s')}{N_i^t(s,a)}$$
and confidence radius 
\begin{align}
    d_i^t(s,a) \coloneqq \sqrt{\frac{2|\mathcal{S}|\ln(2|\mathcal{S}||\mathcal{A}|N\frac{t^4}{\delta})}{\max\{1,N^t_i(s,a)\}}}
\end{align}
where $\delta > 0$ is a user defined constant. 
We can now define the ball $B^t$ of possible values of $P^*$ as 
\begin{align}
    B^t = \{P \ | \ \Vert{P}_i(s,a,\cdot) - \hat{P}_i^t(s,a,\cdot) \Vert_1 \leq d_i^t(s,a) \ \forall i,s,a \}
\end{align}
In particular, $B_i^t$ is the ball of possible values of $P_i^*$ at episode $t$ for some particular arm $i$. It is not difficult to prove that with high probability, $P_i^*$ will lie in the ball $B_i^t$ (see Proposition \ref{prob:confidencebound} in Section \ref{sec:theoreticalresults}). Therefore, throughout the paper, whenever we refer to valid transition matrices, it implies that the transition matrices lie in $B^t$. As there are only two states, the confidence region can be simplified further. $\forall P \in B^t$, we have that for $\forall i,s,a$,
\begin{align*}
&\Vert{P}_i(s,a,\cdot) - \hat{P}_i^t(s,a,\cdot) \Vert_1 \leq d_i^t(s,a) \\
&\implies | {P}_i(s,a,1) - \hat{P}_i^t(s,a,1) | + |{P}_i(s,a,0) - \hat{P}_i^t(s,a,0)| \leq d_i^t(s,a) \\
&\implies | {P}_i(s,a,1) - \hat{P}_i^t(s,a,1) | + | (1-{P}_i(s,a,1)) - (1- \hat{P}_i^t(s,a,1))| \leq d_i^t(s,a) \\
&\implies | {P}_i(s,a,1) - \hat{P}_i^t(s,a,1) | \leq \frac{d_i^t(s,a)}{2}
\end{align*}
Similarly, 
$| {P}_i(s,a,0) - \hat{P}_i^t(s,a,0) | \leq \frac{d_i^t(s,a)}{2}$.
Let us further define:
$$P_i^{+,t}(s,a,1) = \min \left\{1, \hat{P}_i^t(s,a,1) + \frac{d_i^t(s,a)}{2}\right\}$$
$$P_i^{-,t}(s,a,1) = \max\left\{0, \hat{P}_i^t(s,a,1) - \frac{d_i^t(s,a)}{2}\right\}$$
$$P_i^{+,t}(s,a,0) = \max\left\{0, \hat{P}_i^t(s,a,0) - \frac{d_i^t(s,a)}{2}\right\}$$
$$P_i^{-,t}(s,a,0) = \min\left\{1, \hat{P}_i^t(s,a,0) + \frac{d_i^t(s,a)}{2}\right\}$$
as the upper confidence and lower confidence bounds on the transition matrices. Note that $P_i^{+,t}$ is representative of optimistic transitions, as we are overestimating the probability of ending up in a good state. Let us also define the following: 
$\omega_{1,i}^t = P_i^{-,t}(1,1,1) - P_i^{+,t}(0,1,1), \omega_{2,i}^t = P_i^{-,t}(1,0,1) - P_i^{+,t}(0,0,1)$, 
$\eta_{1,i}^t = P_i^{+,t}(1,1,1) - P_i^{-,t}(0,1,1), 
\eta_{2,i}^t = P_i^{+,t}(1,0,1) - P_i^{-,t}(0,0,1)$.
Here, $\omega_{1,i}^t$ denotes the lower bound on the difference in transition probability matrices of arm $i$ when it is pulled and goes from state 1 to state 1 and from 0 to 1, respectively. 
$\omega_{2,i}^t$ denotes lower bound on difference in transition probability matrices when arm is not pulled. $\eta_{1,i}^t, \eta_{2,i}^t$ represents upper bounds of the same quantities. 
We make the following assumption throughout our work. 
\begin{assume}\label{assume:eta}
    $\exists t_0 < T$ that for $\forall t > t_0$,  $\eta_{1,i}^t , \eta_{2,i}^t < 1 \ \forall i \in [N]$
\end{assume}
Assumption ~\ref{assume:eta} discusses the estimated gap in two different transition probabilities. Note that as we are talking about probabilities, that gap can be at most 1. Our assumption is that by $t_0$ episodes, the confidence radius $d_i^t(\cdot \ ,\cdot)$ shrinks enough to make the gap strictly less than 1. This is not a strong assumption, as more and more times the state-action pairs are visited, the confidence region keeps shrinking as the radius is inversely proportional to $N_i^t(s,a)$. 
We provide empirical values of $t_0$ in Table ~\ref{tab:constants} in Section ~\ref{sec:experiments}. Overall, from our simulations, we see that $t_0$ remains low (<40).

\subsection{\ouralgo}

We now define our algorithm \ouralgo \  as follows. For each episode $t$, calculate the estimated reward of each arm $i$ as {$\mu_i^t = f(P_i^{+,t},1) - f(P_i^{+,t},0)$}. 
Then the probability distribution over arms being chosen $\pi^t$ is given by
\begin{align}\label{eqn:our_policy}
    \pi_i^t = \frac{g(\mu_i^t)}{\sum_j g(\mu_j^t)}
\end{align}
It is very important to observe that we can also take any other $P_i^t \in B_i^t$ while estimating the reward (line 5 of Algorithm ~\ref{algo}). As a matter of fact, our regret proofs hold for any set of valid transition matrices. \ouralgo \ takes the upper bound on transition matrices as they are indicative of the upper bound of transitioning to the good state, and we wish to be optimistic about each arm's capability of transitioning to the good state. If an arm is more likely to transition to the good state, then it would lead to the arm being in the good state with a higher probability in the steady state, which is what we wish to achieve. We then sample K arms from this probability distribution without replacement. When $K=1$, the probability of an arm $i$ being chosen is $Pr_i^t(K) = Pr_i^t(1) = \pi_i^t$, and when $K>1$, the probability of arm $i$ being chosen will be $Pr_i^t(K) \geq Pr_i^t(1) = \pi_i^t$. As the probability of being pulled is greater than that of our defined fair policy Equation (\ref{eqn:fairness_K=1}), \ouralgo \ is still fair based on the defined fairness. The complete pseudo code for \ouralgo \ is given in Algorithm ~\ref{algo}.

\begin{algorithm}
\caption{\ouralgo}
\label{algo}
\KwData{No. of arms, budget K, time horizon H, no. of episodes T, merit function g($\cdot$)}
$N_i^t(s,a,s') \gets 0 \ \forall t,i,s,a,s'$\;
\For{$t=1$ to $T$}{
    \For{$i=1$ to $N$}{
        Estimate $\hat{P}_i^t(s,a,s')$ and $d_i^t(s,a)$ using $N_i^t(s,a,s') \ \forall s,a $\;
        $\mu_i^t \gets f(P_i^{+,t},1) - f(P_i^{+,t},0)$\;
    }
    $\pi^t \gets \left(\pi_i^t \ \forall i  \ | \pi_i^t = \frac{g(\mu_i^t)}{\sum_j g(\mu_j^t)} \right)$ \;
    \For{$h=1$ to $H$}{
    Sample K arms from $\pi^t$\;
    Pull the sampled K arms\;
    Observe transitions $(s,a,s')$ for all arms\;
    }
    Update counts $N_i^t(s,a,s') \ \forall i,s,a,s'$\;
}
\end{algorithm}

\section{Theoretical Results}
\label{sec:theoreticalresults}
We now provide the bound for fairness regret of \ouralgo. First, we show through the following lemma that if we take any set of valid transition matrices, define the corresponding reward, and then define the corresponding probability distribution $\pi^t$ (Equation ~\ref{eqn:our_policy}), then for arm $i$, both the states (0 and 1) are visited \emph{and} both the actions (pull and no pull) are taken in some interval $G_i$.
\begin{lem}\label{lem:G}
    For arm $i$, take any $P_i^t \in B_i^t$, define $\mu_i^t = f(P_i^t,1) - f(P_i^t,0)$. Define a policy using Equation ~(\ref{eqn:our_policy}). Then, $\exists G_i < T $ such that $N_i^{t+G_i}(s,a) - N_i^t(s,a) > 0 \ \forall s,a$. In other words, after every $G_i$ episodes, arm $i$ has all its state-action pairs $(s,a)$ visited at least once.
\end{lem}
\begin{proof}
Since $g(\cdot)$ is a non-decreasing merit function we have: 

$\frac{g(-1)}{Ng(1)} \leq \pi_i^t \leq Pr_i^t(K) = \lambda < 1$ unless $K=N$. For $K=N$, the fairness regret Equation (\ref{dfn:regret_def}) will trivially be linear with respect to time. Now for an arm $i$, let $A_i$ be the event that only one single state is visited in episode $t$ (of length H) and 
$B_i$ be the event that only one action is taken for the entire episode $t$. If the initial state of arm $i$ is $s_0$, then,
\begin{align*}
&\mathbb{P}(A_i|s_0 = 0) = \prod_{h=1}^H \left((1-Pr_i^t(K))P_i^*(0,0,0) + Pr_i^t(K) P_i^*(0,1,0)\right)\\ 
&\leq (1-\epsilon)^H \tag*{($\epsilon \leq P_i^*(s,a,s) \leq 1 - \epsilon$)}\\
&\mathbb{P}(A_i|s_0 = 1) = \prod_{h=1}^H \left((1-Pr_i^t(K))P_i^*(1,0,1) + Pr_i^t(K) P_i^*(1,1,1)\right)\\ 
&\leq (1-\epsilon)^H\\
&\mathbb{P}(A_i) = \mathbb{P}(A_i|s_0 = 0) + \mathbb{P}(A_i|s_0 = 1) \leq 2(1-\epsilon)^H\\
&\mathbb{P}(B_i) = \prod_{h=1}^H Pr_i^t(K) + \prod_{h=1}^H (1-Pr_i^t(K)) \leq \lambda^H +\left(1-\frac{g(-1)}{Ng(1)}\right)^H = B_i^0 
\end{align*}

Therefore, 
$\mathbb{P}(A_i \cup B_i) \leq 2(1-\epsilon)^H + B_i^0 - 2B_i^0(1-\epsilon)^H$. Now,
define $\psi^H_i$ as the probability that all state action pairs are visited in an episode, then 
$$\psi_i^H = 1 - Prob(A_i \cup B_i) \geq 1 - \left(2(1-\epsilon)^H + B_i^0 - 2B_i^0(1-\epsilon)^H\right)$$
and thus we get an upper limit on $G_i = 1/\psi_i^H$
\end{proof}

\begin{defn}\label{defn:eta/omega}
    Define $G = \max_i\{G_i\}, \eta = \max_{i,t>t_0}\{\eta_{1,i}^t,\eta_{2,i}^t\} , \\ \omega = \max_{i,t>t_0}\{\omega_{1,i}^t,\omega_{2,i}^t\}$
\end{defn}
\begin{prop}
\label{prob:confidencebound}
    Given $\delta>0$ and $t>1$, $Prob(P^* \in B^t) \geq 1 - \frac{\delta}{t^4}$
\end{prop}
The proposition states that the set of true transition matrices $P^*$ lies in our confidence region $B^t$ with high probability. The proof follows directly from Proposition 6.1 of \cite{wang2023optimistic}. We in next theorem show that if we take any set of valid transition matrices and define the corresponding reward, then the difference between our estimated reward and true reward is upper bounded.  

\begin{thm}\label{thm:mu}
    For all $t > t_0$, take any $P_i^t \in B_i^t$, define $\mu_i^t = f(P_i^t,1) - f(P_i^t,0)$ and $\mu_i^* = f(P_i^*,1) - f(P_i^*,0)$  according to Equation ~(\ref{defn:reward}). Then for $t > t_0$ and $\forall i \in [N]$, $$|\mu_i^t - \mu^*_i| \leq \\ \frac{d_i^t(1,1) + 2d_i^t(0,1) + d_i^t(1,0) + 2d_i^t(0,0)}{(1-\eta)(1-\omega)}$$
\end{thm}
\begin{proof}
We have that for $t>t_0$, 
$\omega_{1,i}^t < P_i^t(1,1,1) - P_i^t(0,1,1) < \eta_{1,i}^t < 1$ and 
$\omega_{2,i}^t < P_i^t(1,0,1) - P_i^t(0,0,1) < \eta_{2,i}^t < 1$.
So, we get
$$\frac{P_i^t(0,1,1)}{1-\omega_{1,i}^t} < f(P_i^t,1) < \frac{P_i^t(0,1,1)}{1-\eta_{1,i}^t}$$
$$\frac{P_i^t(0,0,1)}{1-\omega_{2,i}^t} < f(P_i^t,0) < \frac{P_i^t(0,0,1)}{1-\eta_{2,i}^t}$$
For some arm $i$ in episode $t$ with $t>t_0$, we drop notation $i,t$ by letting $f(1) = f(P_i^t, 1), f(0) = f(P_i^t, 0), f^*(1) = f(P_i^*, 1), f^*(0) = f(P_i^*, 0), \mu = f(P_i^t,1) - f(P_i^t,0), \mu^* = f(P_i^*,1) - f(P_i^*,0), \omega_{1,i}^t = \omega_1, \omega_{2,i}^t = \omega_2, \eta_{1,i}^t = \eta_1, \eta_{2,i}^t = \eta_2$.
Then
\begin{align*}
    |\mu - \mu^*| &= \lvert(f(1) - f(0)) - (f^*(1) - f^*(0))\rvert \\
    &\leq |f(1) - f^*(1)| + |f(0) - f^*(0)| \\
    &\leq \left\lvert\frac{P(0,1,1)}{1-\eta_1} - \frac{P^*(0,1,1)}{1-\omega_1}\right\rvert + \left\lvert\frac{P(0,0,1)}{1-\eta_2} - \frac{P^*(0,0,1)}{1-\omega_2}\right\rvert \\
\end{align*}
\text{The first term $\left|\frac{P(0,1,1)}{1-\eta_1} - \frac{P^*(0,1,1)}{1-\omega_1}\right|$, can be bounded as} \\
\begin{align*}
     &= \left|\frac{P(0,1,1) - P^*(0,1,1) -\omega_1 P(0,1,1) +\eta_1 P^*(0,1,1)}{(1-\eta_1)(1-\omega_1)}\right| \\
    &= \left|\frac{P(0,1,1) - P^*(0,1,1) -\omega_1 (1-P(0,1,0)) +\eta_1 (1 - P^*(0,1,0))}{(1-\eta_1)(1-\omega_1)}\right| \\
    &= \left|\frac{P(0,1,1) - P^*(0,1,1) + \omega_1 P(0,1,0) - \eta_1 P^*(0,1,0)) + \eta_1 - \omega_1}{(1-\eta_1)(1-\omega_1)}\right| \\
    &\leq \frac{|P(0,1,1) - P^*(0,1,1)| + |\omega_1 P(0,1,0) - \eta_1 P^*(0,1,0)) + \eta_1 - \omega_1|}{(1-\eta_1)(1-\omega_1)} \\
\end{align*}
Now, $|\omega_1 P(0,1,0) - \eta_1 P^*(0,1,0)) + \eta_1 - \omega_1|$ can be bounded by
\begin{align*}
     &|\omega_1 P(0,1,0) - \omega_1 P^*(0,1,0) + \omega_1 P^*(0,1,0) - \eta_1 P^*(0,1,0)) + \eta_1 - \omega_1| \\ 
    &= |\omega_1 (P(0,1,0) - P^*(0,1,0)) - P^*(0,1,0)(\eta_1 - \omega_1) + \eta_1 - \omega_1| \\  
    &\leq | \omega_1 (P(0,1,0) - P^*(0,1,0)) + (\eta_1 - \omega_1)| \\ 
    &\leq |\omega_1| |(P(0,1,0) - P^*(0,1,0)| + |\eta_1 - \omega_1| \\
    &\leq |(P(0,1,0) - P^*(0,1,0)| + |\eta_1 - \omega_1|
\end{align*}
Also,    $|\eta_1 - \omega_1| = |(P^+(1,1,1) - P^-(0,1,1)) - (P^-(1,1,1) - P^+(0,1,1))| \leq |P^+(1,1,1) - P^-(1,1,1)| + |P^+(0,1,1) - P^-(0,1,1)| =  d(1,1) + d(0,1)$.

Therefore the first term, $\left|\frac{P(0,1,1)}{1-\eta_1} - \frac{P^*(0,1,1)}{1-\omega_1}\right|$, is bounded by:
\begin{align*}
     &\leq \frac{|P(0,1,1) - P^*(0,1,1)| + |(P(0,1,0) - P^*(0,1,0)| + d(1,1) + d(0,1) 
    }{(1-\eta_1)(1-\omega_1)} \\ 
    &= \frac{\Vert P(0,1,\cdot) - P^*(0,1,\cdot) \Vert_1 + d(1,1) + d(0,1) }{(1-\eta_1)(1-\omega_1)}\\
    &\leq \frac{d(0,1) + d(1,1) + d(0,1) }{(1-\eta_1)(1-\omega_1)} = \frac{d(1,1) + 2d(0,1)}{(1-\eta_1)(1-\omega_1)} \\ 
\end{align*}
We can similarly bound the second term. Going back to the original equation, we get\\
$$|\mu_i^t - \mu^*_i| \leq \frac{d_i^t(1,1) + 2d_i^t(0,1) }{(1-\eta_{1,i}^t)(1-\omega_{1,i}^t)} + \frac{d_i^t(1,0) + 2d_i^t(0,0)}{(1-\eta_{2,i}^t)(1-\omega_{2,i}^t)}$$
With $\eta = \max_{i,t>t_0}\{\eta_{1,i}^t,\eta_{2,i}^t\}$, $\omega = \max_{i,t>t_0}\{\omega_{1,i}^t,\omega_{2,i}^t\}$, \\
$$|\mu_i^t - \mu^*_i| \leq \frac{d_i^t(1,1) + 2d_i^t(0,1) + d_i^t(1,0) + 2d_i^t(0,0)}{(1-\eta)(1-\omega)} $$
\end{proof}
The next theorem states that if we take any set of valid transition matrices, define the corresponding reward, and then define the corresponding probability distribution $\pi^t$ (Equation ~\ref{eqn:our_policy}), then our policy incurs a sublinear fairness regret when $K=1$.
\begin{thm}\label{thm:regret}
    Take any $P_i^t \in B_i^t$, define $\mu_i^t = f(P_i^t,1) - f(P_i^t,0) \\ \forall i \in [N]$. Define a policy using Equation ~(\ref{eqn:our_policy}). Then for $T > t_0$, fairness regret of this policy when $K=1$ is $FR^T = \mathcal{O}\left( \frac{L \sqrt{GT\ln(8N\frac{T}{\delta})}}{\gamma (1-\eta)(1-\omega)} \right)$ with probability at least $1-\delta$.
\end{thm}
 \begin{proof}
     Fairness Regret $FR_1^T$ for $t>t_0$ when $P^* \in B^t$ can be bounded as:
\begin{align*}
    fr^t &= \sum_{i=1}^N   |\pi_i^t - \pi_i^*| = \sum_{i=1}^N  \left\lvert\frac{g(\mu_i^t)}{\sum_{j=1}^N g(\mu_j^t)} - \frac{g(\mu_i^*)}{\sum_{j=1}^N g(\mu_j^*)}\right\rvert \\
    &=\sum_{i=1}^N \left\lvert\frac{g(\mu_i^t)\sum_{j=1}^N g(\mu_j^*) - g(\mu_i^*)\sum_{j=1}^N g(\mu_j^t)}{\sum_{j=1}^N g(\mu_j^t)\sum_{j=1}^N g(\mu_j^*)} \right\rvert\\
    &\leq \frac{2 \sum_{i=1}^N|g(\mu_i^t) - g(\mu_i^*)|}{\sum_{j=1}^N g(\mu_j^t)} \tag*{(Adding and subtracting $g(\mu_i^*)\sum_{j=1}^N g(\mu_j^*)$ in the numerator)}\\
    &\leq  2\frac{ \sum_{i=1}^N|g(\mu_i^t) - g(\mu_i^*)|}{N \gamma}\leq \frac{2 L}{N \gamma}\sum_{i=1}^N|\mu_i^t - \mu_i^*|\\
    &\leq  \frac{2 L}{N \gamma }\sum_{i=1}^N \left ( \frac{d_i^t(1,1) + 2d_i^t(0,1) + d_i^t(1,0) + 2d_i^t(0,0)}{(1-\eta)(1-\omega)} \right ) \\
    &=  \frac{2 L}{N \gamma (1-\eta)(1-\omega) } \sum_{i=1}^N \left ( d_i^t(1,1) + 2d_i^t(0,1) + d_i^t(1,0) + 2d_i^t(0,0) \right ) \\
\end{align*}
Next, we can bound $\frac{2 L}{N \gamma (1-\eta)(1-\omega) }   \sum_{i=1}^N d_i^t(s,a)$ as:
\begin{align*} \\
    &\frac{2L \sqrt{2|\mathcal{S}|\ln(2|\mathcal{S}||\mathcal{A}|N\frac{t^4}{\delta})}}{N \gamma (1-\eta)(1-\omega)} \sum_{i=1}^N   \sqrt{1 / \max\{1,N^t_i(s,a)\}} \\
    &\leq \frac{4L \sqrt{2|\mathcal{S}|\ln(2|\mathcal{S}||\mathcal{A}|N\frac{T}{\delta})}}{N \gamma (1-\eta)(1-\omega)} \sum_{i=1}^N \sqrt{1 / \max\{1,N^t_i(s,a)\}} \\
\end{align*}
Summing this term over time, we get:
\allowdisplaybreaks
\begin{align*}
\allowdisplaybreaks
     &\sum_{t=t_0+1}^T \frac{4L \sqrt{2|\mathcal{S}|\ln(2|\mathcal{S}||\mathcal{A}|N\frac{T}{\delta})}}{N \gamma (1-\eta)(1-\omega)} \sum_{i=1}^N \sqrt{1 / \max\{1,N^t_i(s,a)\}} \\ \\
    &= \frac{4L \sqrt{2|\mathcal{S}|\ln(2|\mathcal{S}||\mathcal{A}|N\frac{T}{\delta})}}{N \gamma (1-\eta)(1-\omega)} \sum_{i=1}^N \sum_{t=t_0+1}^T \sqrt{1 / \max\{1,N^t_i(s,a)\}} \\
\end{align*}
We defined $G_i$ as the expected number of episodes in the interval where all state-action pairs for arm $i$ is visited at least once. Then we have, $\sum_{t=t_0+1}^T \sqrt{1 / \max\{1,N^t_i(s,a)\}} \leq \sum_{j=1}^{T/G_i} G_i/\sqrt{j}$. Note that $\sum_{j=1}^T 1/\sqrt{j} \leq 2\sqrt{T}$.
Therefore, $\frac{2 L}{N \gamma (1-\eta)(1-\omega) } \sum_{t=t_0+1}^T \sum_{i=1}^N d_i^t(s,a) = \frac{2L \sqrt{2|\mathcal{S}|\ln(2|\mathcal{S}||\mathcal{A}|N\frac{t^4}{\delta})}}{N \gamma (1-\eta)(1-\omega)} \sum_{i=1}^N \sqrt{G_i T}$. With $G = \max_i G_i$, we get $FR_1^T = $
\begin{align*}
&\sum_{t=1}^T \sum_{i \in [N]} \left | \pi_i^* - \pi_i^t \right | = \sum_{t=1}^{t_0} \sum_{i \in [N]} \left | \pi_i^* - \pi_i^t \right | + \sum_{t={t_0+1}}^T \sum_{i \in [N]} \left | \pi_i^* - \pi_i^t \right |\\
 &\leq Nt_0 + \frac{48L \sqrt{2G|\mathcal{S}|T\ln(2|\mathcal{S}||\mathcal{A}|N\frac{T}{\delta})}}{ \gamma (1-\eta)(1-\omega)}
 \end{align*}
As $|\mathcal{S}| = |\mathcal{A}| = 2$, we get $FR_1^T \leq N t_0 + \frac{96L \sqrt{GT\ln(8N\frac{T}{\delta})}}{\gamma (1-\eta)(1-\omega)}$.

Note that Fairness Regret when $P^* \notin B^t$ is $FR_2^T \leq N\sqrt{T}$ with probability greater than $(1-\delta)$ ($\because Prob(P^* \in B^t) \geq 1 - \frac{\delta}{t^4}$).

Total Regret = $FR^T = FR_1^T + FR_2^T \leq  N t_0 + N\sqrt{T}+\frac{96L \sqrt{GT\ln(8N\frac{T}{\delta})}}{\gamma (1-\eta)(1-\omega)} $ with probability greater than $1-\delta$
\end{proof}

As \ouralgo \ is defined using $P_i^{+,t}$, which is a valid transition matrix, we get the following corollary
\begin{cor}
    For $t > t_0$, the fairness regret of \ouralgo \ when $K=1$ is $FR^T = \mathcal{O}\left( \frac{L \sqrt{GT\ln(8N\frac{T}{\delta})}}{\gamma (1-\eta)(1-\omega)} \right)$  with probability at least $1-\delta$.
\end{cor}
For $K>1$, the biggest challenge is to define an appropriate notion of fairness regret. One approach could be to take assumptions on $P^*$ and $g(\cdot)$ so that $\pi^*_i \leq \frac{1}{K}; \forall i \in [N]$ which gives 
$\hat{FR}^T = \sum_{t=1}^T \sum_{i \in [N]} \left | \hat{\pi}_i^* - \hat{\pi}_i^t \right|$
where $\hat{\pi}_i = K \times \pi_i$. This new fairness regret $\hat{FR}^T$ would only differ by a factor of $K$ to $FR^T$ (Equation ~\ref{dfn:regret_def}), and hence the sublinearity of Theorem ~\ref{thm:regret} would still hold.

\section{Experimental Section} \label{sec:experiments}
We now validate the efficacy of \ouralgo \ across different domains.
\subsection{Domains}
The three domains selected come from different values of true transition probabilities $P^*$. 

\emph{Synthetic Dataset}: Set the $P_i^*(s,a,s')$ values uniformly in $[0,1]$ $\forall \ i \in [N], \  s,s' \in \mathcal{S}, \ a\in\mathcal{A}$ 

\emph{Synthetic-alternate Dataset}: Similar to synthetic dataset, we set the $P_i^*(s,a,s')$ values uniformly in $[0,1]$ $\forall \ i \in [N], \  s,s' \in \mathcal{S}, \ a\in\mathcal{A}$. In addition, the following two constraints are imposed on $P_i^*(s,a,s')$. (1) Acting is always beneficial, i.e., $P_i^*(s,1,1) \geq P_i^*(s,0,1) \ \forall i \in [N] \ s \in \mathcal{S}$. (2) Starting in a good state is always beneficial, i.e., $P_i^*(1,a,1) \geq P_i^*(0,a,1) \ \forall i \in [N] \ a \in \mathcal{A}$.

\emph{CPAP Dataset}: Continuous positive airway pressure therapy (CPAP) is considered to be a very effective treatment for obstructive sleep apnoea. However, some patients do not necessarily adhere to the therapy. We use the Markov model of CPAP treatment given by ~\citet{kang2016modelling}. We adapt their three-state model into two states in a similar fashion to ~\cite{herlihy2023planning, li2023avoiding}. We combine states 2 and 3 into one state, and set the ratio of improvement of intervention $\alpha_h = 1.1$. 
We set $30\%$ of the total arms as non-adherers (to the therapy). To add heterogeneity to the arms, we add Gaussian noise with mean $0$ and variance $0.1$ to each arm's transition probabilities.  

\subsection{Experimental Setup}
For the experimental evaluation of fairness regret, we compare \ouralgo \ with an "Optimal" baseline, where in each episode $t$, Optimal policy pulls the arms with the $K$ highest value for $\mu_i^t$. Observe that as $\mu_i \in [-1,1]$, we can trivially claim that in Conditions ~\ref{cond:gamma} and ~\ref{cond:L}, $\gamma = g(-1)$ and $L=\max_{\mu_1,\mu_2 \in [-1,1]} \frac{g(\mu_1) - g(\mu_2)}{\mu_1 - \mu_2}$ hold. We use $\delta=0.01$ and set the merit function $g(\mu) = e^{c\mu}$, implying that $\gamma=e^{-c}$ and $L=ce^c$. We set $c=3$ and provide additional experiments with different values of $c$ in the Appendix. In line with our non-degeneracy assumption on $P^*$, we clip the transition probabilities in the range $[\epsilon, 1-\epsilon]$ with $\epsilon=0.01$ for all three datasets. We include an empirical comparison with FaWT-Q~\cite{li2022efficient} in the appendix. The results are averaged over 30 independent runs with different seed values.

We run the experiments for $T$=10k episodes and $H$=200 timesteps per episode for a total of $T \times H = 2\times10^6$ timesteps. We set the initial state of each arm randomly per episode. In accordance with our setting, we use a relatively high time horizon $H$, as our reward formulation and subsequent analysis are centered around steady state, and a longer time horizon would make it easier to achieve that steady state. For $K>1$ cases, we use the same definition of fairness regret (Equation ~\ref{dfn:regret_def}) as for $K=1$ cases. The code is available at \href{https://github.com/rchiso/MF-RMAB}{https://github.com/rchiso/MF-RMAB}. The experiments were executed on a Xeon processor with 64 GB RAM.

\subsection{Results}
\begin{figure}[t!]
\centering
    \includegraphics[width=0.3\textwidth]{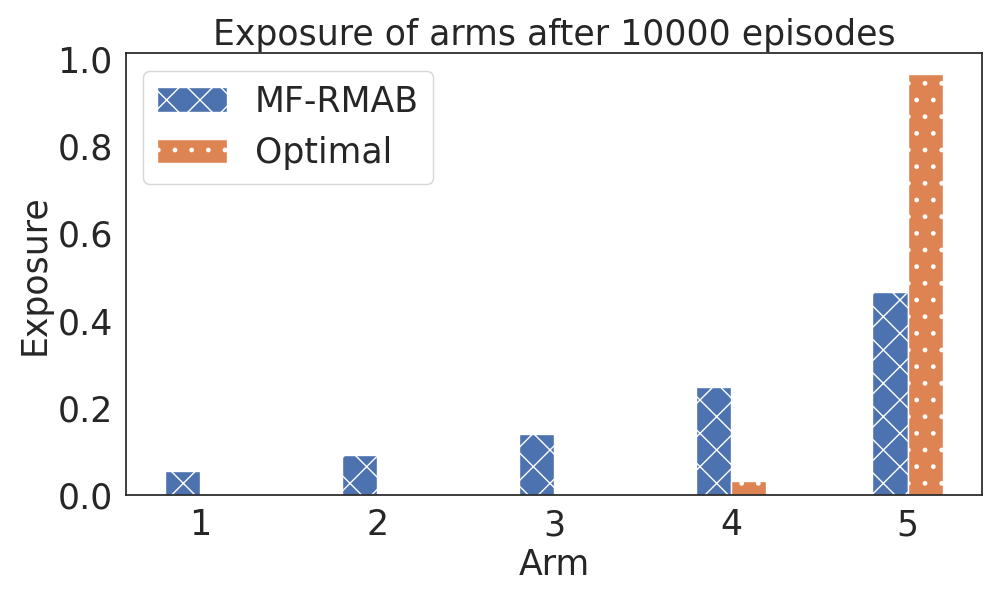}
    \caption{Exposure of arms under Optimal and \ouralgo \ on Synthetic dataset after 10k episodes for $N=5$, $K=1$. The arms are arranged in increasing order of their rewards.}
    \label{fig:exposure_synthetic}
\end{figure}
\begin{figure*}
\begin{subfigure}{0.23\textwidth}
\includegraphics[width=\textwidth]{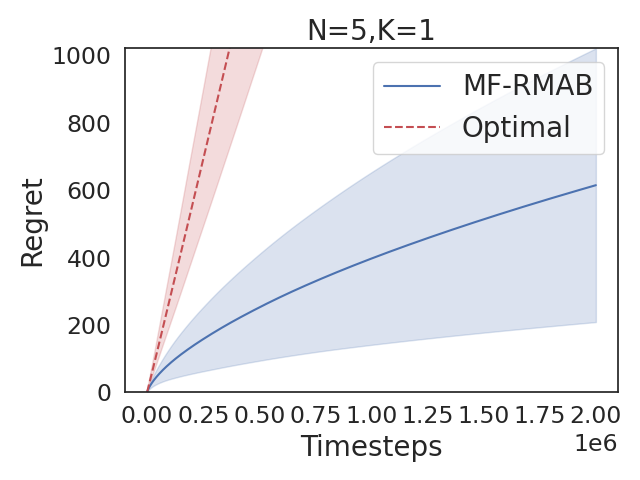}
\end{subfigure}
\begin{subfigure}{0.23\textwidth}
\includegraphics[width=\textwidth]{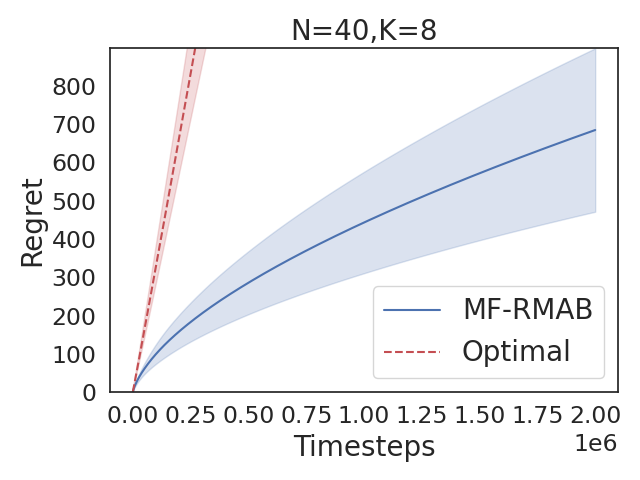}
\end{subfigure}
\begin{subfigure}{0.23\textwidth}
\includegraphics[width=\textwidth]{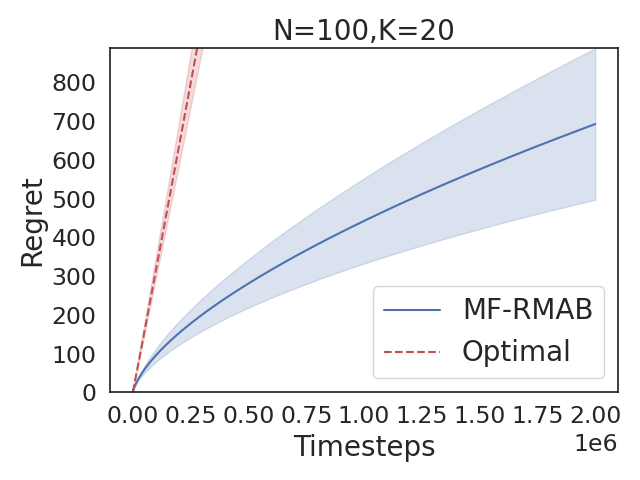}
\end{subfigure}
\begin{subfigure}{0.23\textwidth}
\includegraphics[width=\textwidth]{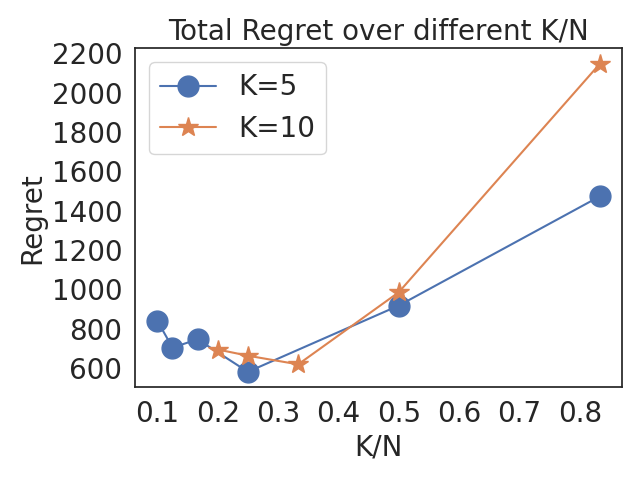}
\end{subfigure}
\caption{The first three plots show Regret vs. Time for different $K$ and $N$ settings on Synthetic dataset. The last plot shows the Regret with different $K/N$ values for $T \times H = 2\times 10^6$ timesteps.}
\label{fig:syn}
\end{figure*}
\begin{figure*}
\begin{subfigure}{0.23\textwidth}
\includegraphics[width=\textwidth]{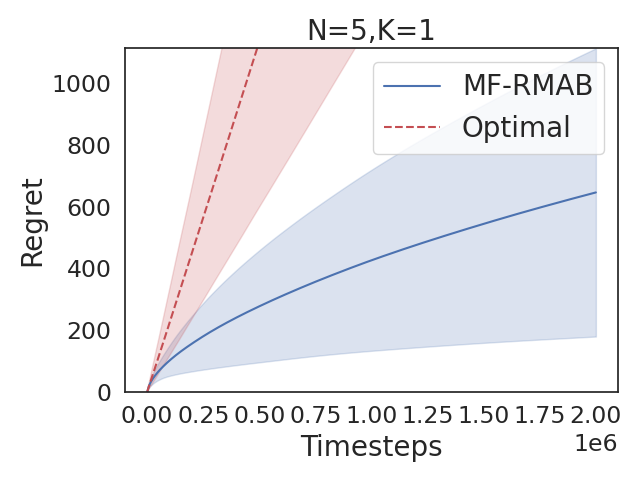}
\end{subfigure}
\begin{subfigure}{0.23\textwidth}
\includegraphics[width=\textwidth]{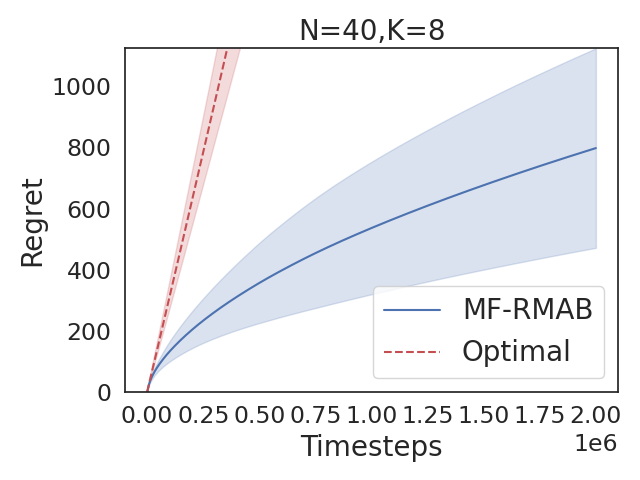}
\end{subfigure}
\begin{subfigure}{0.23\textwidth}
\includegraphics[width=\textwidth]{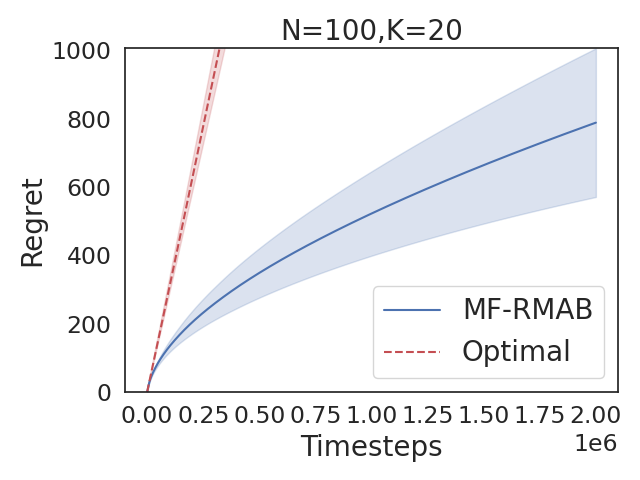}
\end{subfigure}
\begin{subfigure}{0.23\textwidth}
\includegraphics[width=\textwidth]{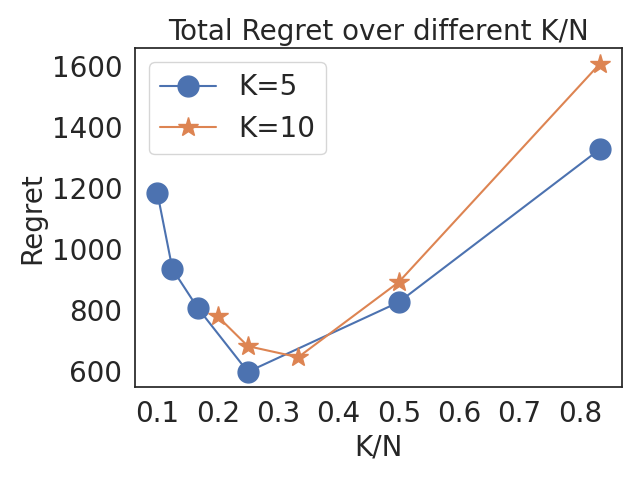}
\end{subfigure}
\caption{The first three plots show Regret vs. Time for different $K$ and $N$ settings on Synthetic-alternate dataset. The last plot shows the Regret with different $K/N$ values for $T \times H = 2\times 10^6$ timesteps.}
\label{fig:syn_alt}
\end{figure*}
\begin{figure*}
\begin{subfigure}{0.23\textwidth}
\includegraphics[width=\textwidth]{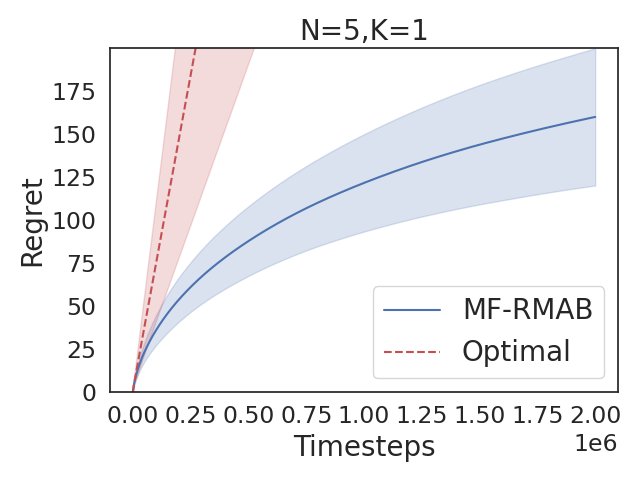}
\end{subfigure}
\begin{subfigure}{0.23\textwidth}
\includegraphics[width=\textwidth]{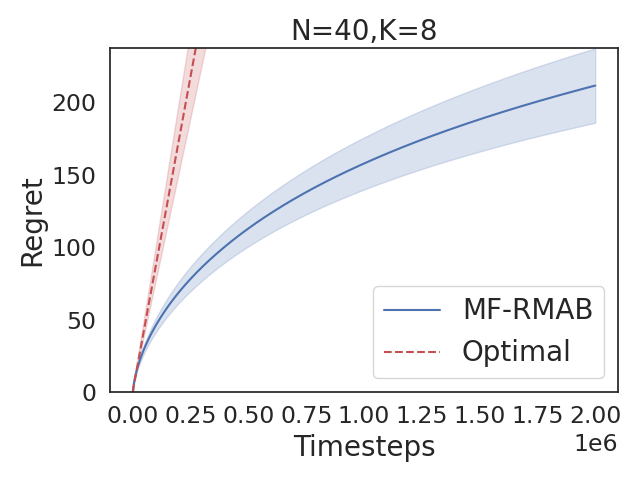}
\end{subfigure}
\begin{subfigure}{0.23\textwidth}
\includegraphics[width=\textwidth]{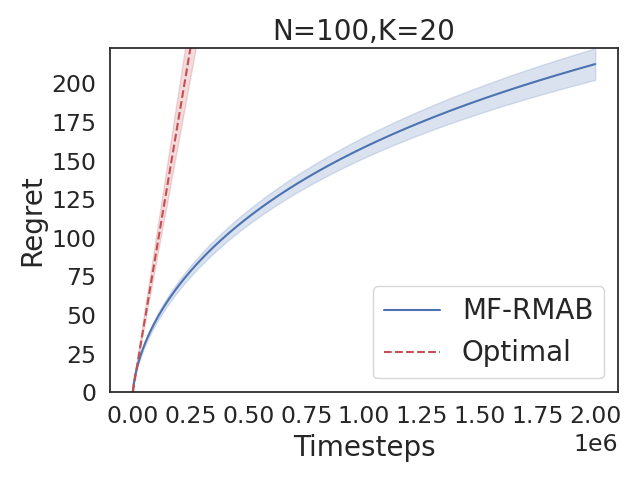}
\end{subfigure}
\begin{subfigure}{0.23\textwidth}
\includegraphics[width=\textwidth]{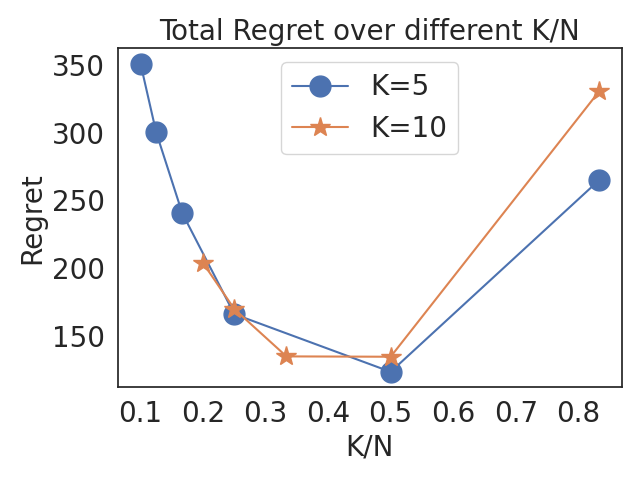}
\end{subfigure}
\caption{The first three plots show Regret vs. Time for different $K$ and $N$ settings on CPAP dataset. The last plot shows the Regret with different $K/N$ values for $T \times H = 2\times 10^6$ timesteps.}
\label{fig:cpap}
\end{figure*}
Figure ~\ref{fig:exposure_synthetic} shows the exposure arms get after 10k episodes on Synthetic dataset. Immediately, we can see the need for a fair policy, as Optimal tends to completely ignore sub-optimal arms, while \ouralgo\ gives exposure proportional to the merit of the arms, and ensures fairness. 
\begin{table}
\centering
\begin{tabular}[t!]{@{}ccccccc@{}}
    \toprule
    \textbf{} &
      \multicolumn{2}{c}{\textbf{Syn}} & 
      \multicolumn{2}{c}{\textbf{Syn-alt}} &
      \multicolumn{2}{c}{\textbf{CPAP}}  \\
      & {$G$} & {$t_0$} & {$G$} & {$t_0$} & {$G$} & {$t_0$} \\
      \cmidrule(r){2-3}\cmidrule(lr){4-5}\cmidrule(l){6-7}
    N=5 K=1 & 38 &  8 & 58 & 23 & 57 & 19  \\
    N=10 K=2 & 128 &  9 & 103 & 22 & 84 & 21  \\
    N=20 K=4 & 115 &  14 & 169 & 30 & 89 & 21  \\
    N=40 K=8 & 220 &  15 & 251 & 31 & 106 & 23  \\
    N=100 K=20 & 668 &  13 & 649 & 36 & 118 & 23 \\
    \bottomrule
  \end{tabular}
  \caption{Values of $G$ and $t_0$ on different datasets.}
  \label{tab:constants}
\end{table}
Table ~\ref{tab:constants} shows the empirical values of $t_0$ and $G$ over the three datasets. We can see that while $t_0$ remains mostly consistent across different values of $N$ and $K$, $G$ varies significantly. This is because as $N$ increases, even if $K$ increases proportionally with $N$, the probability of some arm receiving a pull decreases, so $G_i$ can get very large for some arm in the worst case.
As CPAP dataset is more homogeneous across arms as compared to the other two datasets, we see that the values of $G$ and $t_0$ are more uniform.

The first three plots of Figures ~\ref{fig:syn},~\ref{fig:syn_alt} and ~\ref{fig:cpap} show the various trends of fairness regret across the three datasets. We can see that \ouralgo \ incurs a sublinear regret, while Optimal is unable to learn a fair policy and exhibits linear regret. As Synthetic and Synthetic-alternate datasets have a large amount of variance in the transition probabilities, we observe a large variance in regret as well. 

The rightmost plot of Figures ~\ref{fig:syn},~\ref{fig:syn_alt} and ~\ref{fig:cpap} show the variation of total regret $FR^{T}, T=10$k over increasing $\frac{K}{N}$ ratio. We can observe that in Synthetic and Synthetic-alternate datasets, the regret reaches its minimum around $\frac{K}{N} \approx 0.3$, while in CPAP dataset, the minima is around $\frac{K}{N} \approx 0.5$. Therefore, we conclude that increasing $K$ does not necessarily help in learning the transition probabilities faster, and can end up increasing the regret instead.

\section{Conclusion/Future Work}
We introduce exposure fairness to the online RMAB setting in the form of \ourfair. We provide a sublinear bound on fairness regret in the single-pull case, and show that our algorithm \ouralgo \ works admirably even in multiple-pull case. Future work could include formally defining a robust formulation of fairness regret for the multiple-pull case, along with provable sublinear regret bounds. Another possible research direction can be to define Fairness Regret using other possible reward formulations.

\section*{Acknowledgement}
The author Shweta Jain would like to acknowledge the DST grant MTR/2022/000818 for providing the support to carry out this work.



\bibliographystyle{ACM-Reference-Format} 
\bibliography{biblio}

\newpage
\appendix
\onecolumn
\section{MF-RMAB vs FaWT-Q}

We now provide an empirical comparison with FaWT-Q \cite{li2022efficient}. The FaWT-Q algorithm implements the WIQL algorithm of \cite{biswas2021learn} with the extra fairness contraint that each arm is pulled atleast $\eta$ number of times every $L$ timesteps. In accordance with the original work, we fix $\eta=2$ and vary $L=\{15,30,50\}$. For MF-RMAB, we vary the merit function $g(\mu) = e^{c\mu}$ with different values of $c$. Higher values of $L$ and $c$ means a more optimal algorithm, while lower values ensures more fairness. There is no reward penalty for violating the fairness constraints in FaWT-Q. We set $H=200$ and $T=1000$. As \cite{li2022efficient} only provide results on Synthetic-alternate dataset, we run the simulations on the same. The results are averaged over 30 random seeds. The plots are shown in Figure~\ref{fig:benchmark}. 

\begin{figure}[ht!]
\centering
\begin{subfigure}{\textwidth}
\centering
\includegraphics[width=0.5\textwidth]{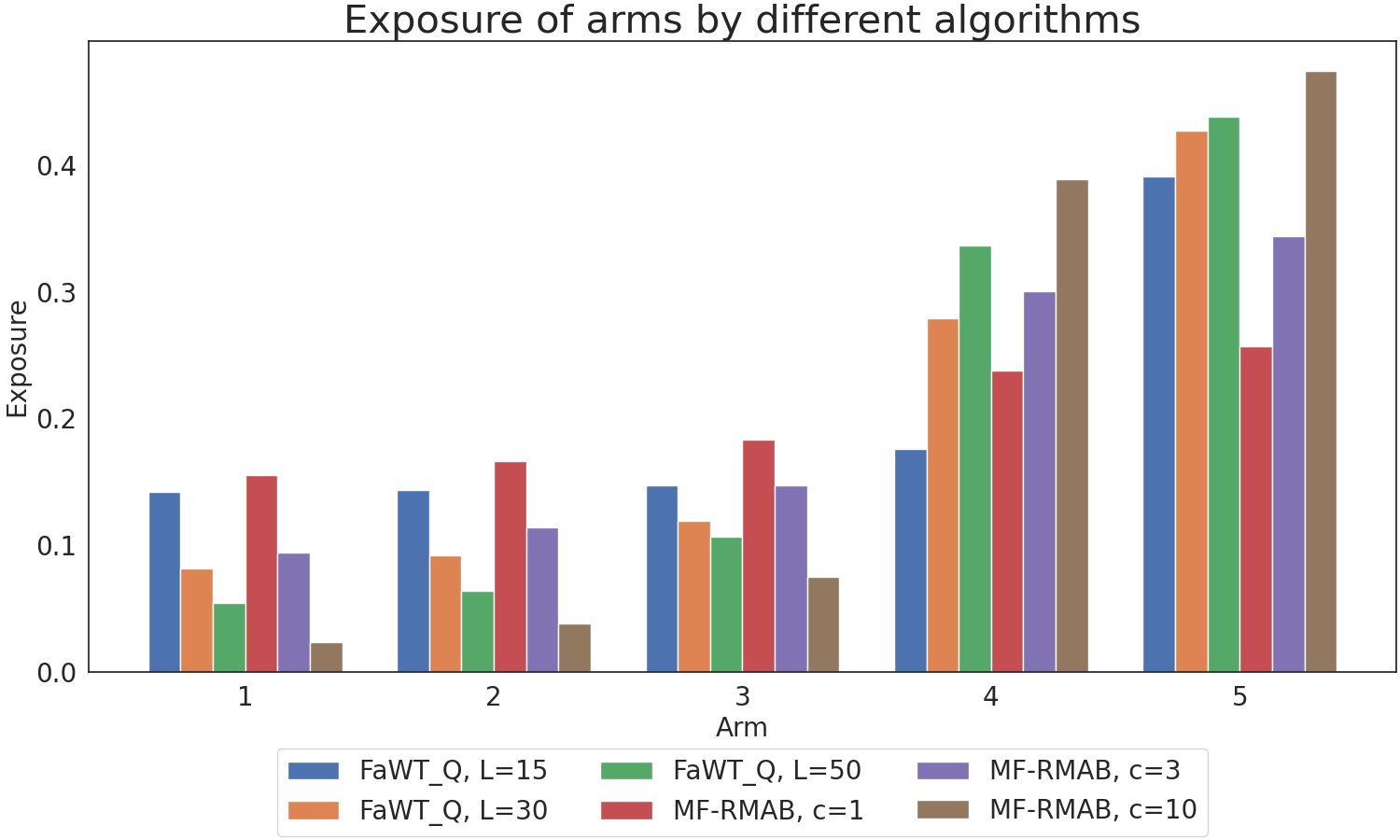}
\includegraphics[width=0.4\textwidth]{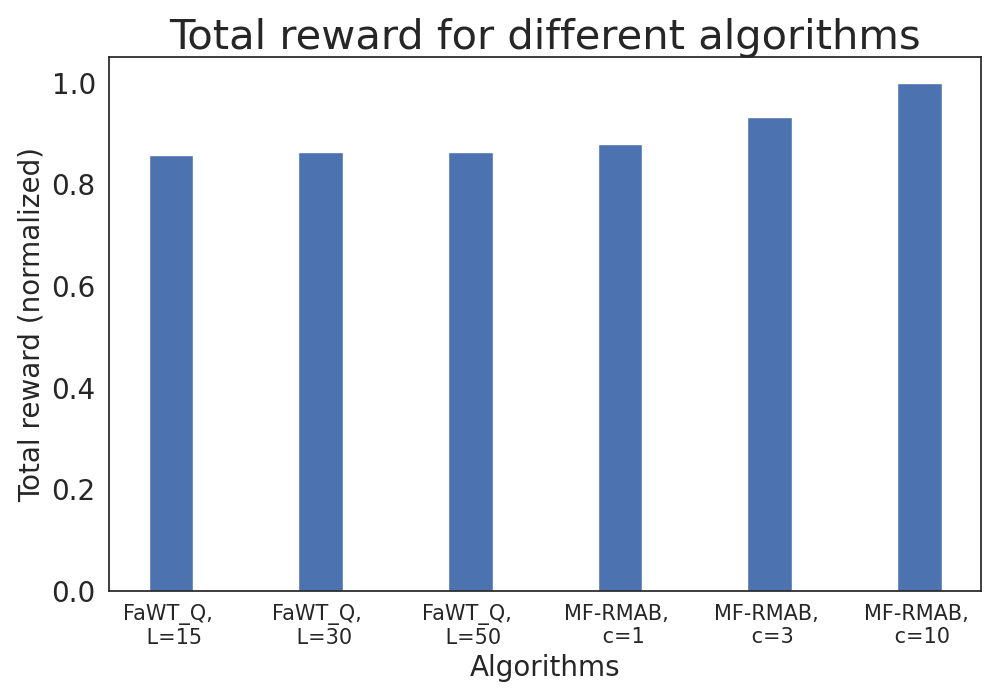}
\caption{$N=5$ and $K=1$}
\label{benchmark1}
\end{subfigure}
\begin{subfigure}{\textwidth}
\centering
\includegraphics[width=0.5\textwidth]{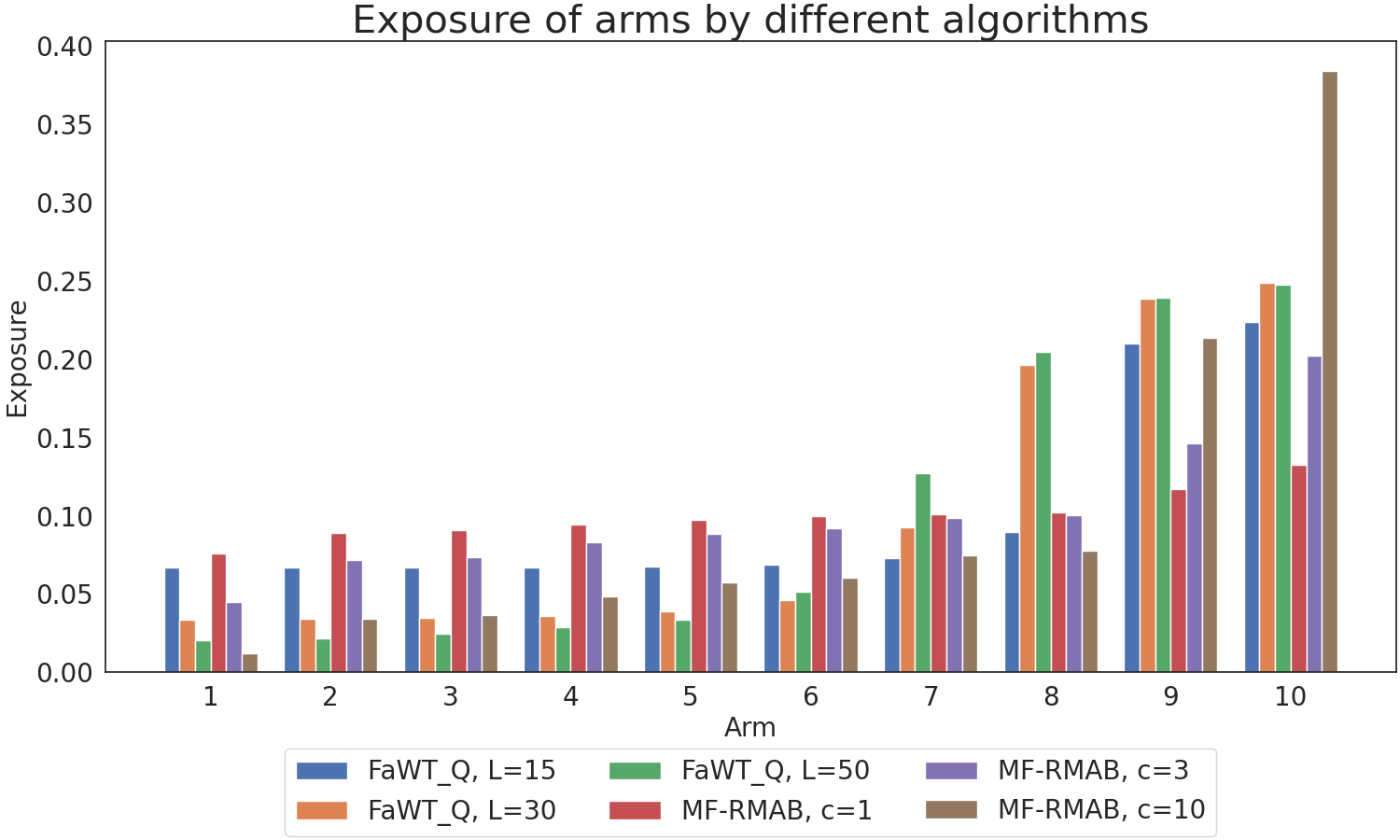}
\includegraphics[width=0.4\textwidth]{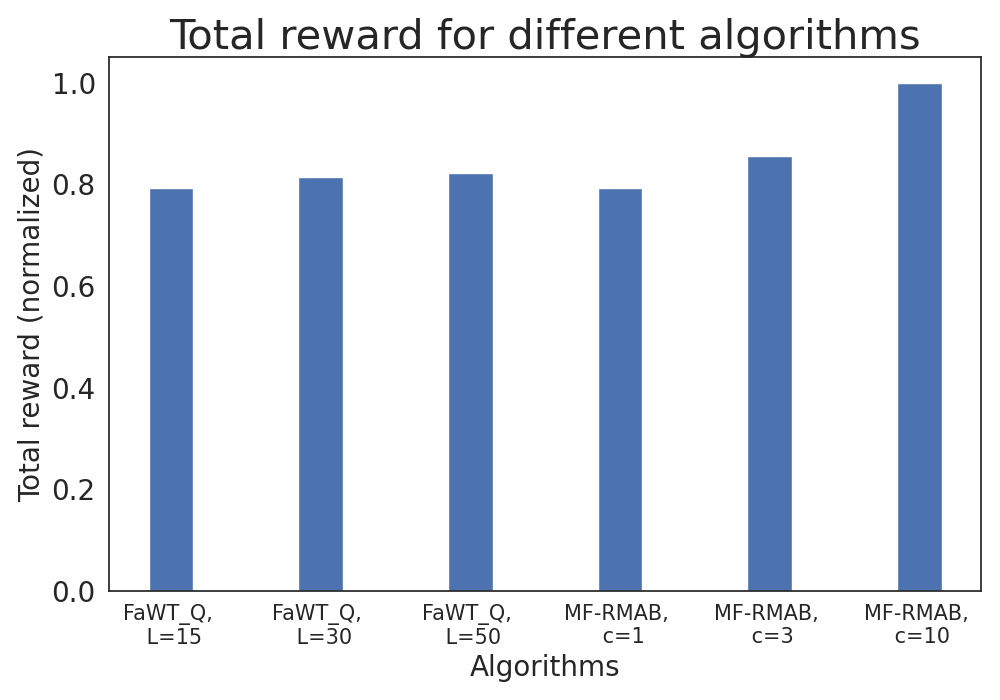}
\caption{$N=10$ and $K=2$}
\label{benchmark2}
\end{subfigure}
\caption{Comparison of MF-RMAB and FaWT-Q with respect to exposure and rewards. The y-axes are normalized to 1.}
\label{fig:benchmark}
\end{figure}

We can see that MF-RMAB provides a more egalitarian experience than FaWT-Q while performing equal or better in terms of reward. FaWT-Q tends to focus on the top few arms and ignores the rest, while MF-RMAB provides exposure to the worst arms as well. A very high value of $c$ makes the algorithm effectively optimal and thus offers little in terms of fairness. 

\newpage

\section{Additional Experiments}
We explore the trends in fairness regret for different values of $c$ in the merit function $g(\mu) = e^{c\mu}$.
\subsection{c=1}
\begin{figure}[!htbp]
    \centering
    \includegraphics[width=0.3\textwidth]{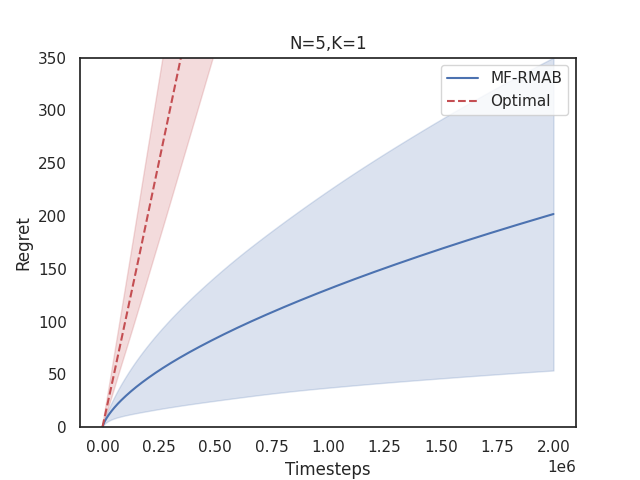}
    \includegraphics[width=0.3\textwidth]{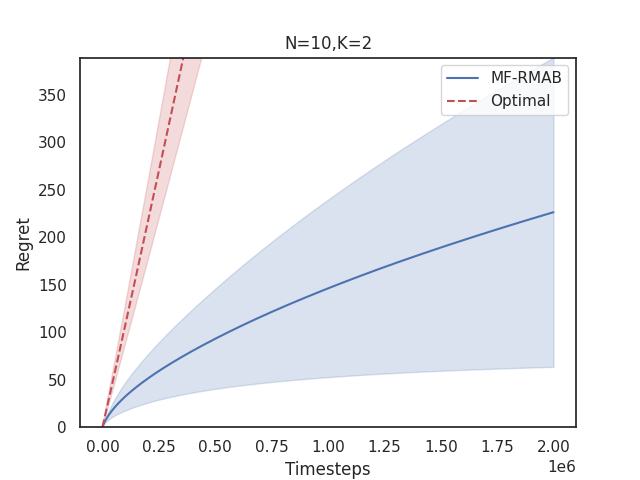}
    \includegraphics[width=0.3\textwidth]{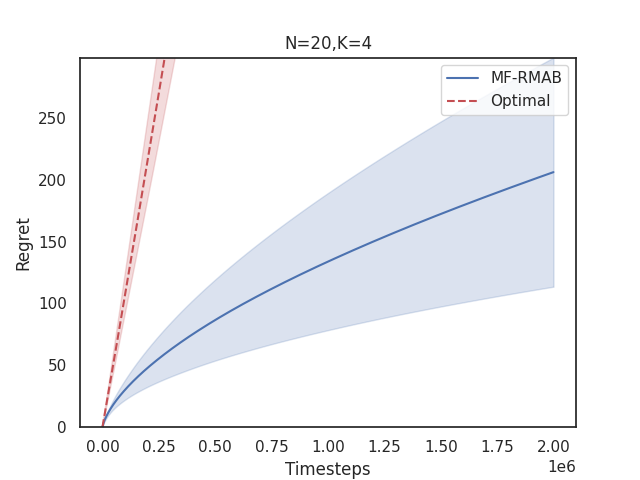}
    \includegraphics[width=0.3\textwidth]{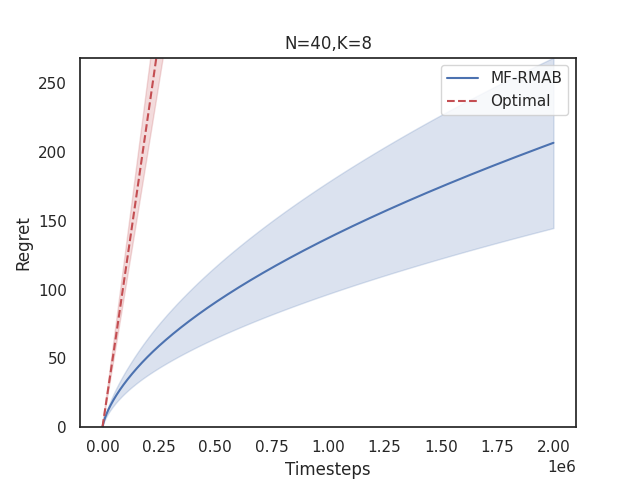}
    \includegraphics[width=0.3\textwidth]{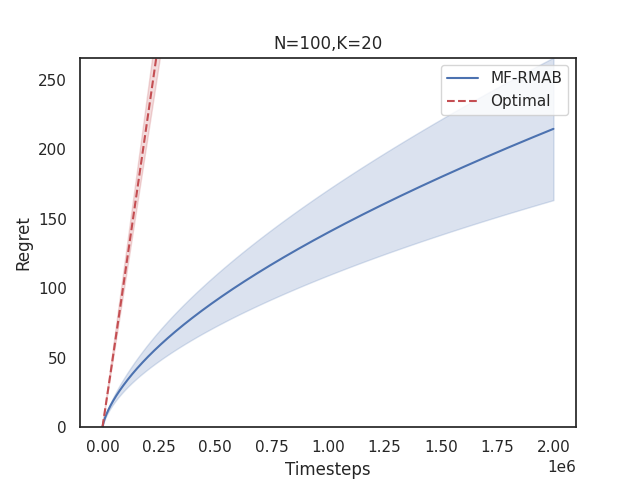}
    \caption{Fairness Regret on Synthetic dataset when $c=1$}
    \label{fig:c1_syn}
\end{figure}
\begin{figure}[!htbp]
    \centering
    \includegraphics[width=0.3\textwidth]{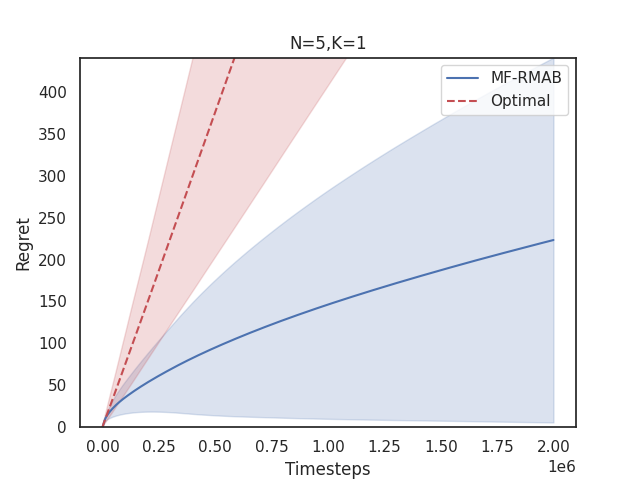}
    \includegraphics[width=0.3\textwidth]{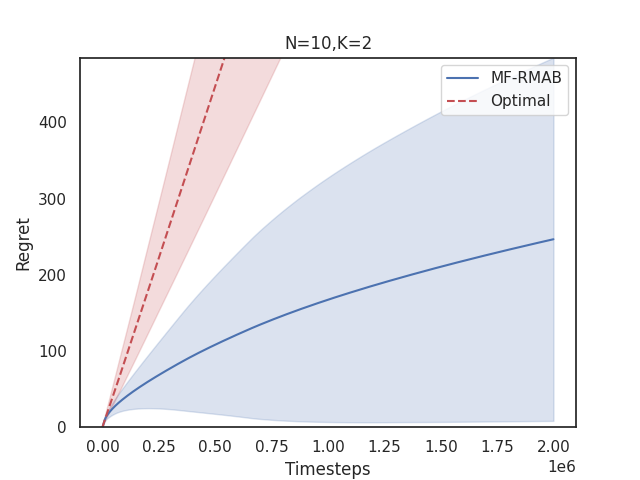}
    \includegraphics[width=0.3\textwidth]{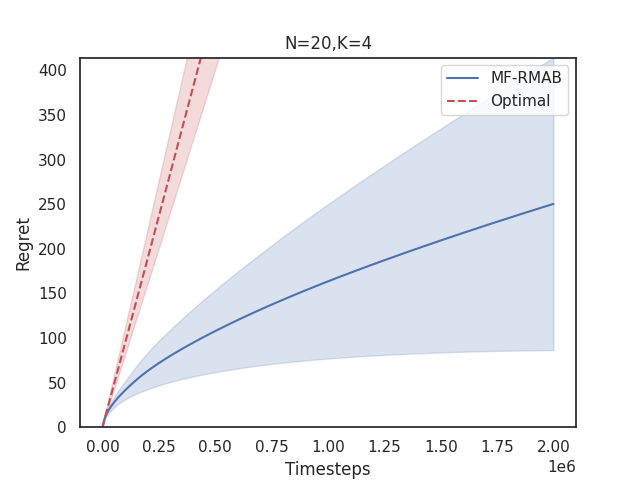}
    \includegraphics[width=0.3\textwidth]{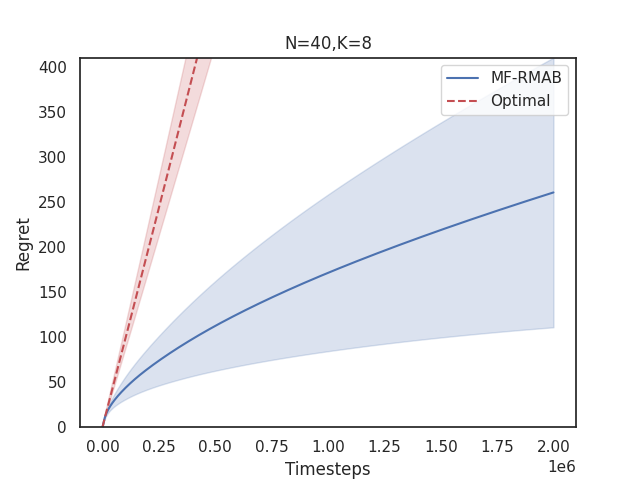}
    \includegraphics[width=0.3\textwidth]{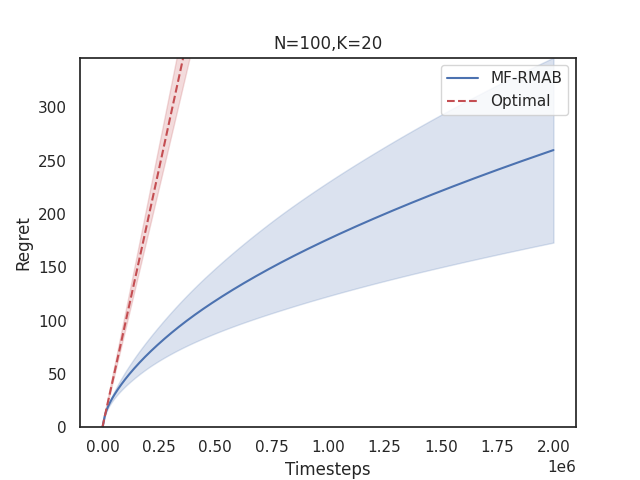}
    \caption{Fairness Regret on Synthetic-alternate dataset when $c=1$}
    \label{fig:c1_syn-alt}
\end{figure}
\begin{figure}[!htbp]
    \centering
    \includegraphics[width=0.3\textwidth]{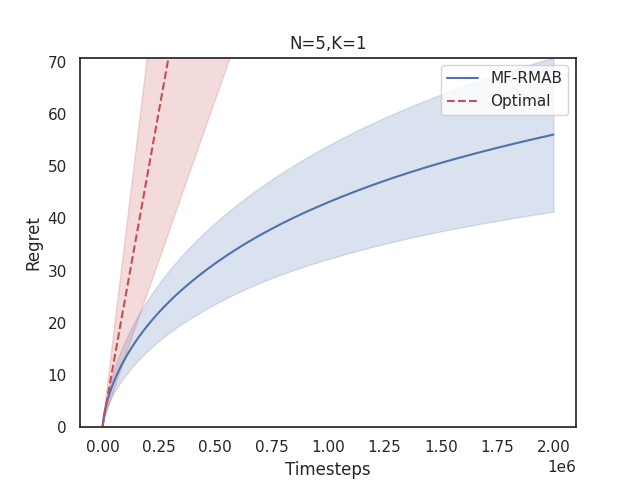}
    \includegraphics[width=0.3\textwidth]{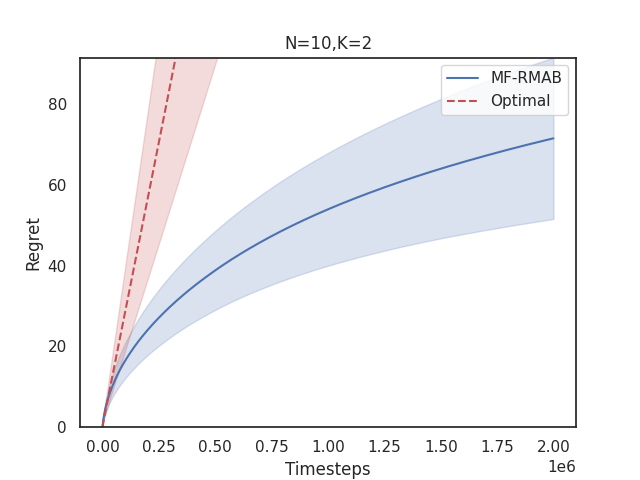}
    \includegraphics[width=0.3\textwidth]{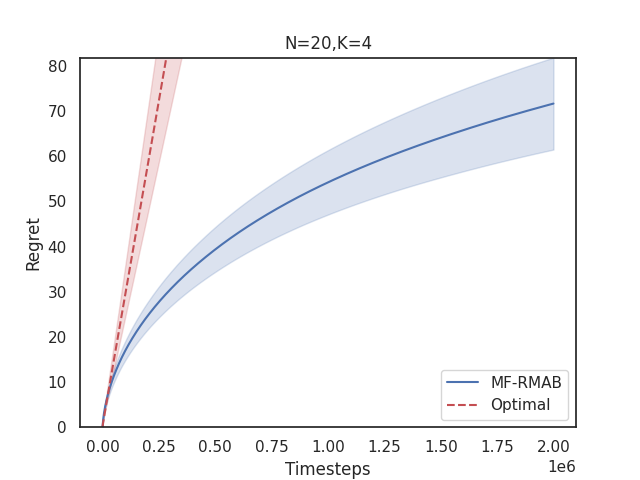}
    \includegraphics[width=0.3\textwidth]{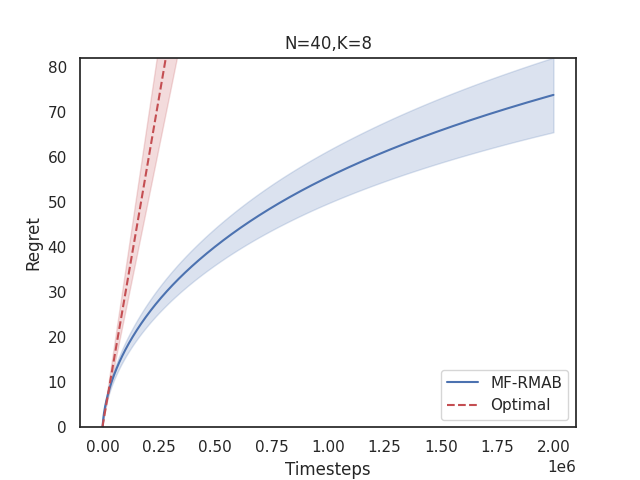}
    \includegraphics[width=0.3\textwidth]{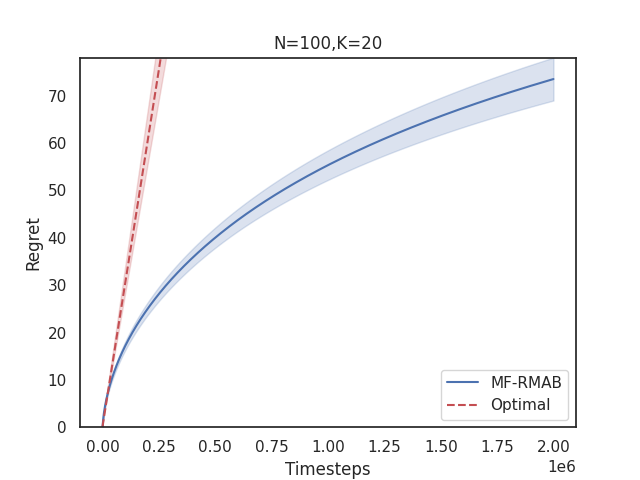}
    \caption{Fairness Regret on CPAP dataset when $c=1$}
    \label{fig:c1_real}
\end{figure}
We observe from Figure \ref{fig:c1_syn}, \ref{fig:c1_syn-alt} and \ref{fig:c1_real} that the value of $c$ directly influences the value of fairness regret. It is interesting to note that the variance has increased. A possible explanation is that a lower value of $c$ leads to less variation of the rewards of the arms, and it becomes more difficult to distinguish one arm from another, leading to higher variance in results.
\newpage
\subsection{c=10}
\begin{figure}[!htbp]
    \centering
    \includegraphics[width=0.3\textwidth]{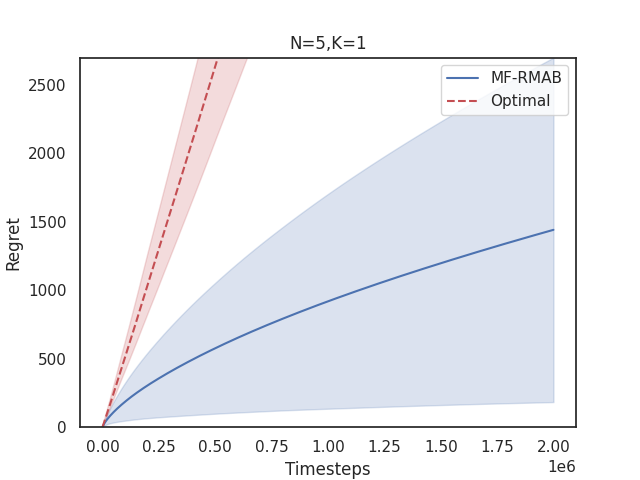}
    \includegraphics[width=0.3\textwidth]{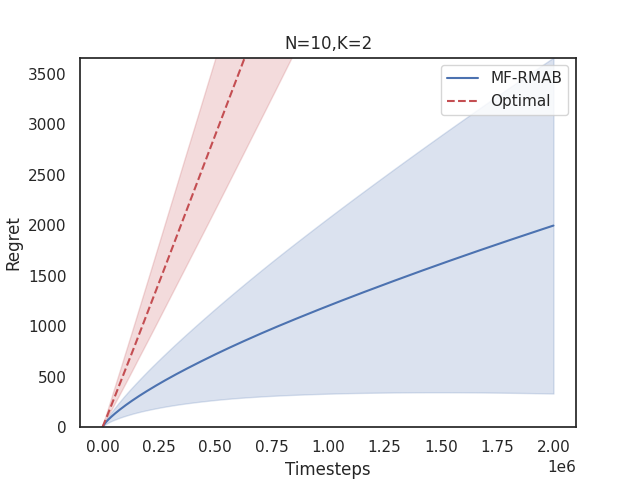}
    \includegraphics[width=0.3\textwidth]{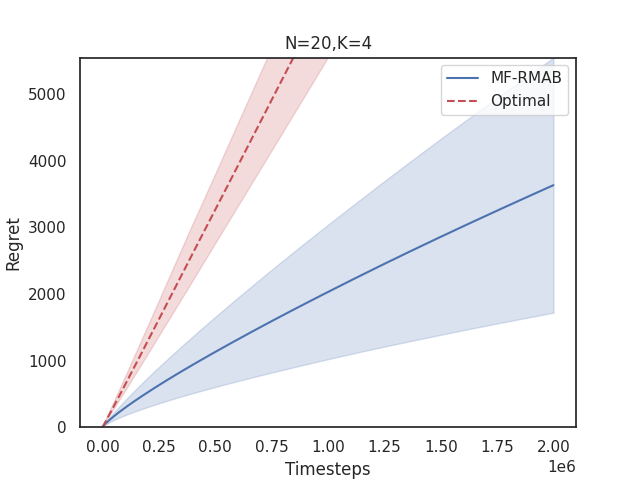}
    \includegraphics[width=0.3\textwidth]{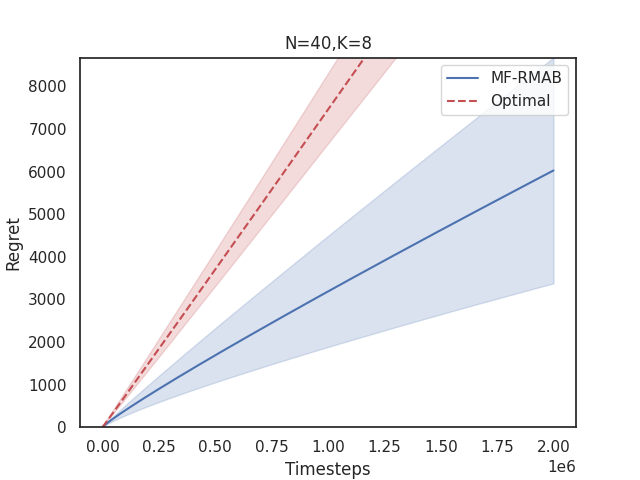}
    \includegraphics[width=0.3\textwidth]{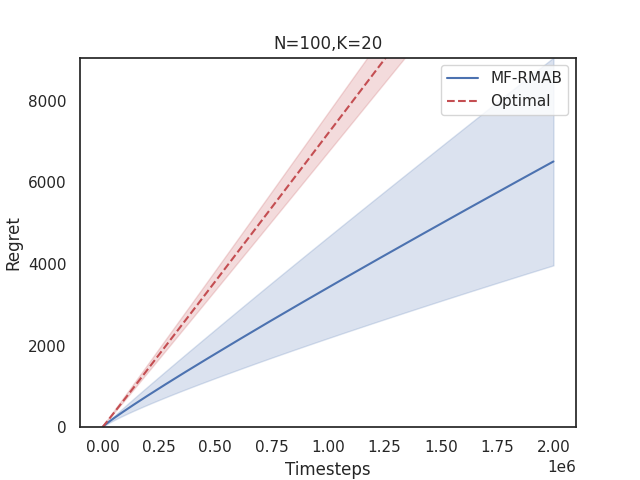}
    \caption{Fairness Regret on Synthetic dataset when $c=10$}
    \label{fig:c10_syn}
\end{figure}
\begin{figure}[!htbp]
    \centering
    \includegraphics[width=0.3\textwidth]{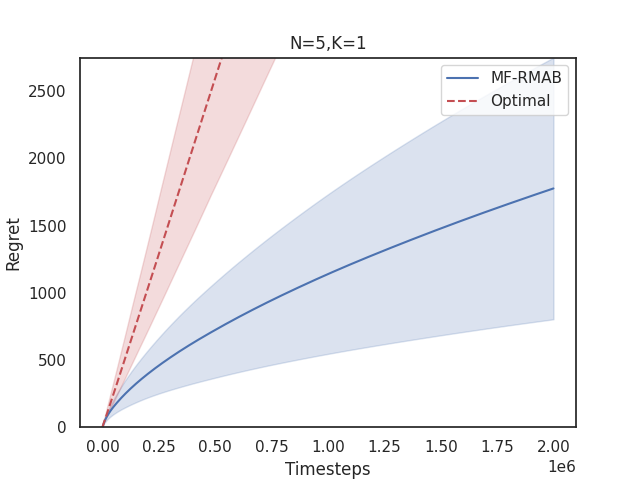}
    \includegraphics[width=0.3\textwidth]{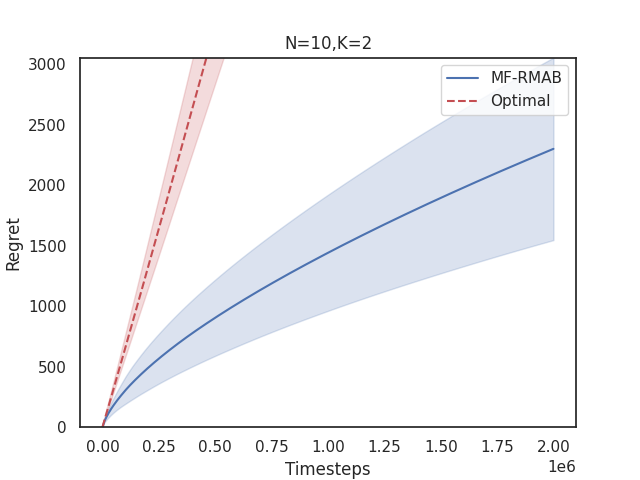}
    \includegraphics[width=0.3\textwidth]{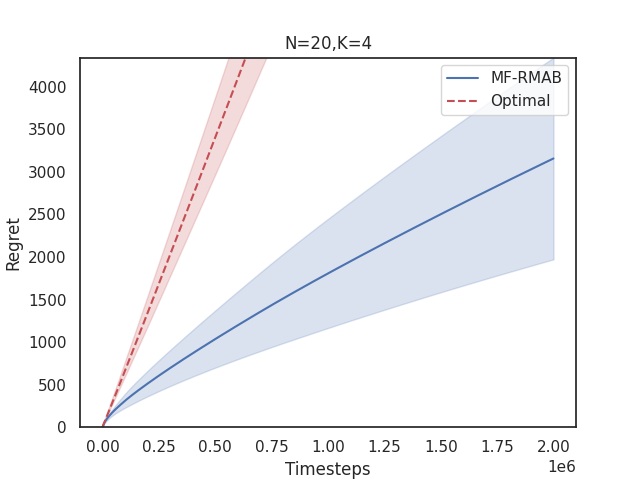}
    \includegraphics[width=0.3\textwidth]{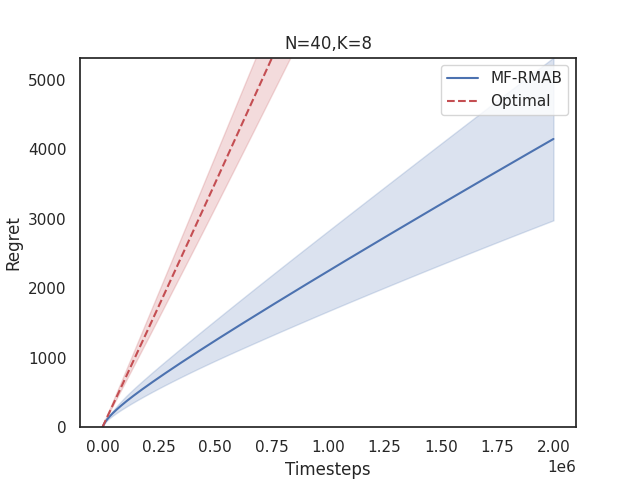}
    \includegraphics[width=0.3\textwidth]{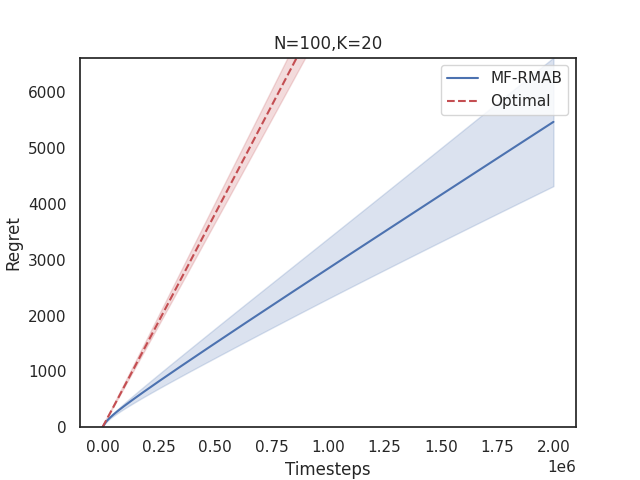}
    \caption{Fairness Regret on Synthetic-alternate dataset when $c=10$}
    \label{fig:c10_syn-alt}
\end{figure}
\begin{figure}[!htbp]
    \centering
    \includegraphics[width=0.3\textwidth]{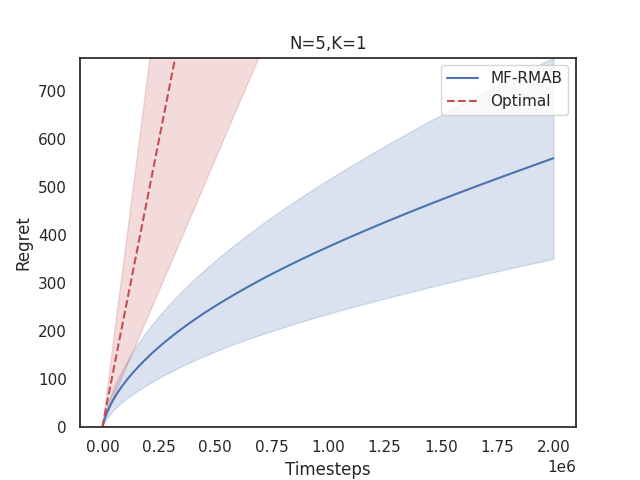}
    \includegraphics[width=0.3\textwidth]{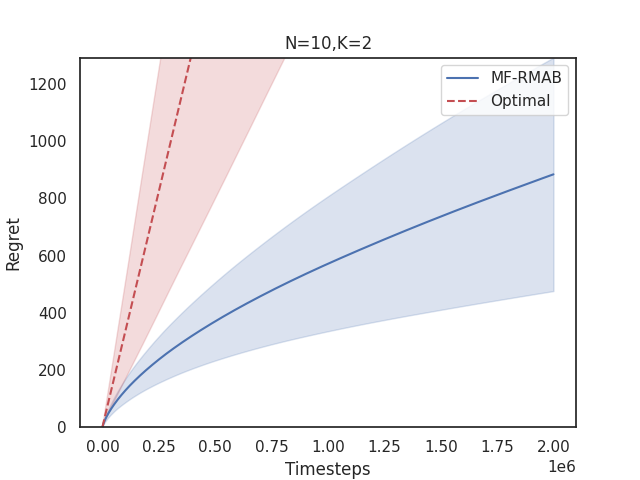}
    \includegraphics[width=0.3\textwidth]{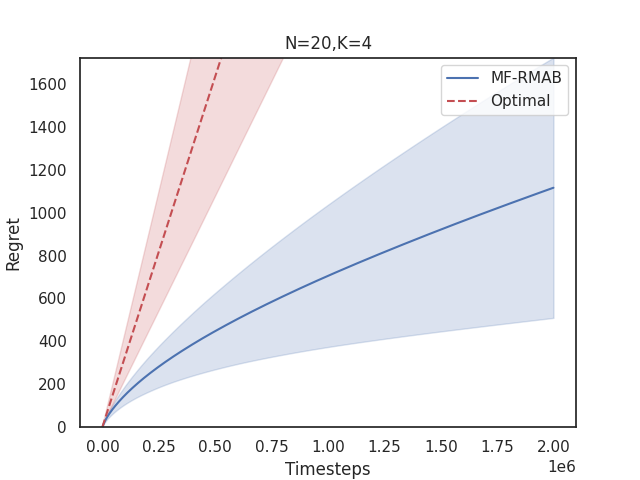}
    \includegraphics[width=0.3\textwidth]{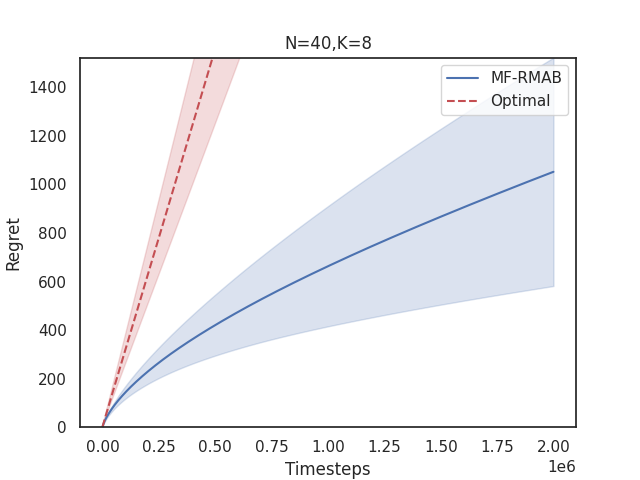}
    \includegraphics[width=0.3\textwidth]{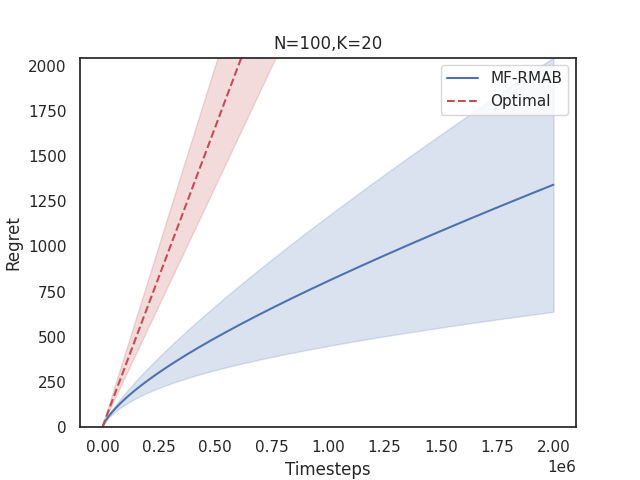}
    \caption{Fairness Regret on CPAP dataset when $c=10$}
    \label{fig:c10_real}
\end{figure}

We observe from Figure \ref{fig:c10_syn}, \ref{fig:c10_syn-alt} and \ref{fig:c10_real} that a higher $c$ leads to higher regret, and the regret curve becomes increasingly more linear. This is because $c$ is essentially a hyper-parameter that calibrates the trade-off between fairness and optimality. $c=0$ ensures uniform fairness across all arms, while a very high value of $c$ would result in a nearly optimal algorithm.

\end{document}